\documentclass{article}
\usepackage[utf8]{inputenc}

\usepackage[preprint]{neurips_2026}
\usepackage{subcaption}
\usepackage{graphicx}
\usepackage{enumitem}

\usepackage[utf8]{inputenc} 
\usepackage[T1]{fontenc}    
\usepackage{hyperref}       
\usepackage{url}            
\usepackage{booktabs}       
\usepackage{amsfonts}       
\usepackage{nicefrac}       
\usepackage{microtype}      
\usepackage{xcolor}         
\usepackage{dblfloatfix}

\usepackage{multirow}
\usepackage{makecell}
\usepackage{amsmath}
\usepackage{graphicx}
\usepackage{amsthm}
\usepackage{amssymb}
\usepackage{subcaption}
\usepackage[most]{tcolorbox}
\usepackage{comment}

\usepackage{tabularx}

\newtcolorbox{promptbox}[2][]{
  enhanced,
  breakable,
  colback=gray!10,
  colframe=black,
  colbacktitle=gray!25,
  coltitle=black,
  fonttitle=\bfseries,
  title={#2},
  boxrule=0.5pt,
  arc=0pt,
  left=10pt,
  right=10pt,
  top=8pt,
  bottom=8pt,
  boxsep=2pt,
  attach boxed title to top left={yshift=-2mm},
  boxed title style={
    boxrule=0.5pt,
    colframe=black,
    colback=gray!25,
    arc=0pt,
    left=4pt,
    right=4pt,
    top=2pt,
    bottom=2pt
  },
  #1
}

\newtheorem{theorem}{Theorem}
\newtheorem{proposition}[theorem]{Proposition}
\newtheorem{corollary}[theorem]{Corollary}

\title{From Correctness to Preference: A Framework for Personalized Agentic Reinforcement Learning
}

\author{%
  Ranxu Zhang$^{1}$\thanks{Equal contribution. Email: zkd\_zrx@mail.ustc.edu.cn} \quad 
  Zeyang Li$^{1}$\footnotemark[1] \quad 
  Jiacheng Huang$^{2}$ \quad 
  Rui Zhang$^{2}$ \quad \\
  \textbf{Xiaozhou Xu}$^{2}$ \quad 
  \textbf{Zhe Sun}$^{2}$ \quad 
  \textbf{Yanyong Zhang}$^{1}$ \quad 
  \textbf{Chao Wang}$^{1}$\thanks{Corresponding author.} \\
  $^{1}$University of Science and Technology of China \quad 
  $^{2}$Alibaba Group
}

\begin{document}

\maketitle

\begin{abstract}

Agentic reinforcement learning (Agentic RL) has achieved strong progress in tasks with clear success signals. 
However, many real-world agent applications require user-conditioned behavior: the same query may call for different planning strategies and tool-use decisions across users. 
This setting raises key challenges: generic rewards cannot capture heterogeneous user preferences, observed behaviors are entangled with conformity effects, and flat memories cannot support personalized skill retrieval. 
To this end, we propose a unified personalized Agentic RL framework that embeds personalization into training-time optimization. 
At its core is \emph{Personalized Anchor Reward-Decoupled Policy Optimization} (\textbf{PARPO}), which decouples generic task-quality rewards from personalized preference rewards and uses user-specific anchors to stabilize learning under heterogeneous reward scales. 
We further introduce a two-stage preference-disentangled reward model and \emph{Preference-Aligned Skill Evolution Graph Memory} (\textbf{PSGM}) for personalized supervision and preference-aligned skill retrieval. 
Together, they form a closed loop of preference identification, policy optimization, and structured skill accumulation. 
Experiments on ETAPP, ETAPP-Hard, and SJAgent show that our framework consistently outperforms strong memory and RL baselines. 
Code and data are included in the supplementary materials.

\end{abstract}

\section{Introduction}
\label{sec:1}

Large language model (LLM)-based agentic reinforcement learning (Agentic RL) has emerged as a powerful paradigm for optimizing LLM agents, achieving strong performance in code generation~\citep{gehring2025rlef, le2022coderl, yang2024swe}, web navigation~\citep{nakano2021webgpt, qi2025webrl}, tool use~\citep{toolrl2025, feng2026retool, skillrl2026}, and long-horizon planning~\citep{tang2026alphaagentevo, peng2026hiper, xi2026agentgymrl}. 
However, most of this progress has been made in \emph{verifiable} settings~\citep{toolrl2025,jin2025searchr,deepseekmath2024}, where policy optimization can rely on a unique ground-truth answer. 
In contrast, in many real-world agent applications---such as e-commerce assistance, travel planning, and daily scheduling~\citep{skarlinski2024language,schmidgall2024agentclinic,ning2025deeptravel,xie2024travelplanner,personalalign2026}---this verifiability breaks down because optimal behavior is user-dependent: the same query may admit multiple plausible trajectories, with the preferred one determined by the user's preferences, habits, and constraints.

Recent work has begun to extend LLM optimization beyond strictly verifiable tasks. 
Non-verifiable or open-ended optimization methods use LLM-based evaluation~\citep{ye2025justice,chan2024chateval,liu2023geval,zheng2023judging}, rubric-based rewards~\citep{RaR2025,liu2025openrubrics}, and other reward construction schemes~\citep{ou2026serl,ye2025self,tan2025process,xu2025direct} to provide supervision when exact answers are unavailable. 
However, these methods typically optimize generic objectives, such as overall quality, helpfulness, or rubric satisfaction, rather than user-conditioned preferences. 
In parallel, personalized agent methods incorporate user profiles, prompting strategies, agent designs, and memory retrieval over historical interactions~\citep{zhang2025personaagent,liang2026learning,cai2025large,wang2025mem,personalalign2026,su2026beyond}. 
While effective for improving user alignment, they largely personalize behavior at inference time and do not directly optimize policies for user-contingent trajectories. 
Therefore, existing methods lack a native training-time optimization framework for personalized agent behavior.

Personalization fundamentally changes the optimization target of Agentic RL. 
For the same query, users may prefer different planning strategies and tool-use decisions; for example, ``plan a one-day trip in Tokyo'' may call for a museum-centered route for one user but an anime-themed route for another. 
Therefore, the agent must move beyond learning a single average-optimal policy and instead learn user-contingent tool-use trajectories. 
As illustrated in Figure~\ref{fig:challenge_overview}, this setting introduces three core challenges. 
\textbf{(C1) Personalized reward ambiguity}: generic rewards mainly capture correctness, task completion, or overall helpfulness, but cannot express how a specific user evaluates the same trajectory, nor can they handle heterogeneous reward scales across users. 
\textbf{(C2) Personalized preference disentanglement}: observed user behaviors are often shaped by both intrinsic interests and external conformity or contextual effects, making individualized preference signals noisy and difficult to identify accurately. 
\textbf{(C3) User-aware memory and skill organization}: existing agent memories are often flat and query-centric, and thus cannot explicitly model or retrieve the structured relations among users, intents, skills, tools, scenarios, and trajectories.

To bridge this gap, we propose a unified personalized Agentic RL framework that embeds personalization into the training-time optimization loop. 
At its core is \emph{Personalized Anchor Reward-Decoupled Policy Optimization} (\textbf{PARPO}), which decouples generic task-quality rewards from personalized preference rewards and uses user-specific anchors to stabilize learning under heterogeneous reward scales. 
PARPO preserves general task competence while enabling user-contingent policy improvement. 
To provide cleaner personalized supervision, we develop a two-stage preference-disentangled reward model that separates intrinsic interests from conformity and contextual effects. 
To support personalized rollout context, we introduce \emph{Preference-Aligned Skill Evolution Graph Memory} (\textbf{PSGM}), an evolving heterogeneous graph memory that organizes users, skills, tools, scenarios, and trajectories for preference-aligned skill retrieval. 
Together, these components form a closed loop of preference identification, personalized policy optimization, and structured skill accumulation. 
Our contributions are summarized as follows:
\begin{itemize}[leftmargin=1.2em,labelsep=0.4em,itemsep=0.05em,topsep=0pt,parsep=0pt,partopsep=0pt]
    \item We formulate personalized Agentic RL for user-conditioned agent tasks, where optimal behavior depends on individual preferences.
    \item We propose \textbf{PARPO}, an anchor-stabilized and reward-decoupled policy optimization method for learning personalized policies under heterogeneous user reward scales.
    \item We introduce a preference-disentangled reward model and \textbf{PSGM} to provide reliable personalized supervision and preference-aligned skill retrieval.
    \item We evaluate our framework on ETAPP~\citep{ETAPP}, ETAPP-Hard, and SJAgent (a real-world industrial agent training scenario from a large Chinese e-commerce platform), showing gains in personalization and procedural quality while maintaining factual and logical quality. 
\end{itemize}

\begin{figure}[t]
    \centering
    \includegraphics[width=0.9\linewidth]{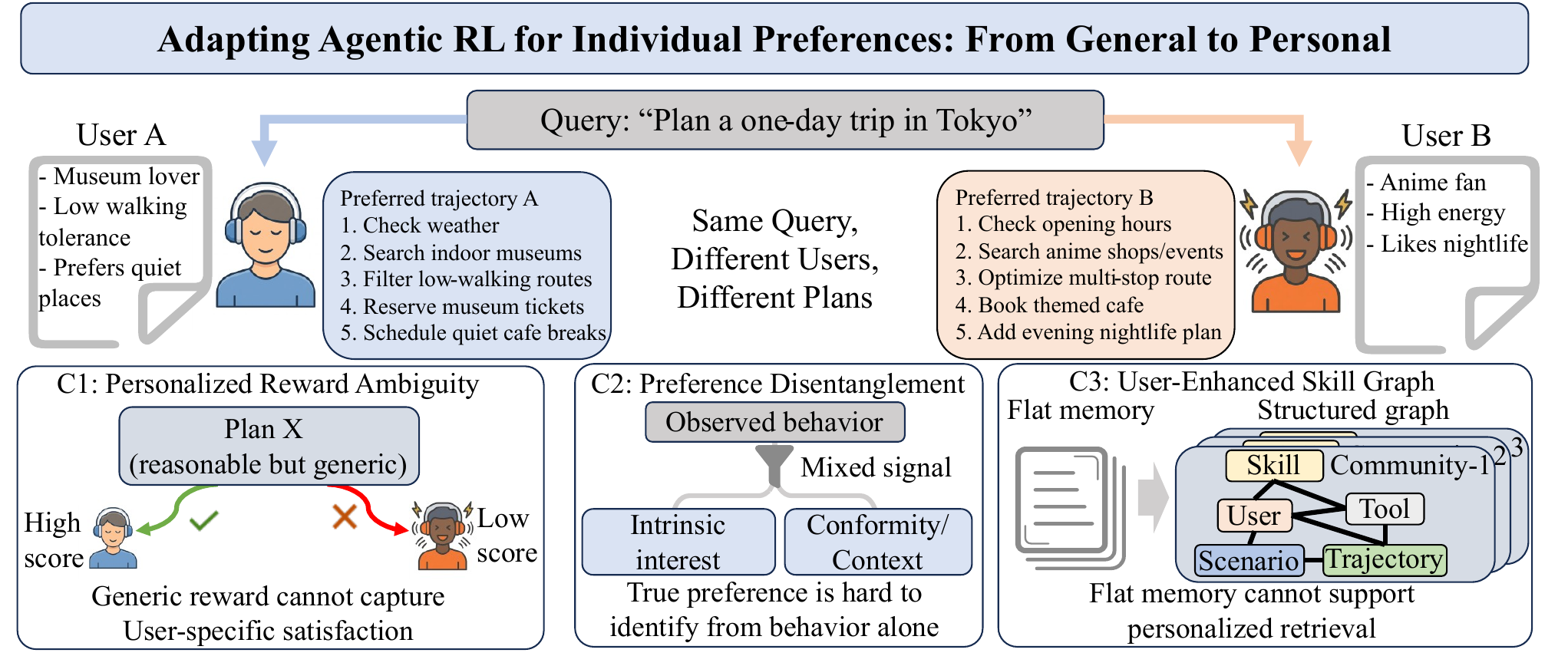}
    \caption{
    Personalization in Agentic RL changes the notion of optimal behavior: the same query may require different plans for different users.
    }
    \label{fig:challenge_overview}

\end{figure}

\section{Related Work}

\noindent\textbf{Agentic RL in Verifiable Settings.}
RL has proven effective for improving LLM agents on verifiable tasks, including Retool, ToolRL, AutoWebGLM, and Search-R1~\citep{feng2026retool,toolrl2025,jin2025searchr,lai2024autowebglm}. GRPO-style methods such as GRPO, DAPO, GSPO, GiGPO, and GDPO improve scalability, sequence-level optimization, and multi-reward stability~\citep{deepseekmath2024,dapo2025,gspo2025,feng2026groupingroup,gdpo2026}. However, these advances remain centered on verifiable settings with correctness or task-success signals, rather than user-specific alignment~\citep{jin2025searchr,skillrl2026,feng2026groupingroup}.

\noindent\textbf{Non-verifiable and Open-ended Optimization.}
RL has also been extended beyond verifiable tasks using LLM-based evaluation, rubric supervision, and learned rewards, e.g., OpenRubrics and Rubrics as Rewards~\citep{liu2025openrubrics,RaR2025}. While these methods support open-ended outputs, they still optimize generic objectives, such as quality, rubric satisfaction, rather than personalized behavior.

\noindent\textbf{Personalization, Preference, and Memory.}
Prior work explores personalization through profiles, memory, and personalized agents, including PersonaAgent, O-Mem, Preference-Aware Memory Update, and Learning Personalized Agents from Human Feedback~\citep{zhang2025personaagent,wang2025mem,prefmemory2025,liang2026learning}. CoPD studies true-interest/conformity entanglement in user behavior~\citep{copd}, while memory- and skill-based agents retrieve skills and user context~\citep{memrl2026,skillrl2026,zhou2026mem,liu2026simplemem}. Most focus on inference-time personalization or omit personalized policy in user-conditioned environments. Our work instead unifies training-time user-conditioned policy optimization, true-preference reward modeling, and personalized skill retrieval.

\noindent A more detailed discussion of relevant prior work is provided in Appendix~\ref{app:related_work}.

\section{Problem Definition}
\label{sec:3}

As discussed in the Introduction, personalization is important in real-world scenarios such as e-commerce and daily planning, since user satisfaction depends not only on whether the task is completed, but also on whether the result aligns with user preferences and whether the overall decision-making experience is satisfactory. Personalized agent behavior can be naturally described as a user-conditioned Markov decision process (MDP):
\[\small
\mathcal{M}=(\mathcal{S},\mathcal{A},\mathcal{P},\mathcal{Q},T,R,\gamma),
\]

where $\mathcal{S}$ denotes the state space, $\mathcal{A}$ denotes the action space, $\mathcal{P}$ denotes the user profile space, $\mathcal{Q}$ denotes the user query space, $T:\mathcal{S}\times\mathcal{A}\to\mathcal{S}$ is the transition function, $R$ is the reward function, and $\gamma$ is the discount factor. For each instance, the user profile $p_u\in\mathcal{P}$ and the user query $q\in\mathcal{Q}$ jointly define the task condition, and the initial state $s_0$ is determined by $(p_u,q)$. The agent then selects actions according to the policy $\pi_\theta(a_t \mid s_t, p_u, q)$, which induces a trajectory $\tau=(s_0,a_0,s_1,a_1,\dots,s_T,a_T)$.

The training objective is to maximize the expected trajectory reward:
\begin{equation}\small
\max_\theta \mathbb{E}_{(u,q)\sim\mathcal{D},\,\tau\sim\pi_\theta(\cdot\mid p_u,q)}\bigl[R(\tau,p_u,q)\bigr].
\end{equation}

The reward consists of two conceptually distinct components: a general-quality reward and a personalized preference reward. The general-quality reward evaluates task completion, logical coherence, and procedural correctness according to the benchmark-specific evaluation protocol, while the personalized preference reward measures how well a trajectory aligns with the target user's preferences. Therefore, the central challenge is that the same query may correspond to different optimal trajectories for different users. Accordingly, a personalized agent needs to optimize both general task quality and preference alignment, while also handling differences in the scale and distribution of personalized rewards across users.



\section{Method}
\label{sec:4}
To address the challenge of personalized decision-making in user-conditioned tasks, we propose a unified Agentic RL framework. As illustrated in Figure~\ref{fig:framework}, the framework operates as a closed loop of personalized retrieval, generation, evaluation, and refinement. Given a user query and profile, the framework first retrieves relevant historical skills from a graph-based memory to form a personalized rollout context. Conditioned on this enriched context, the LLM-based policy interacts with the environment to generate decision-making trajectories. These trajectories are then evaluated by a dual-reward system: a general task-quality reward and a personalized preference reward. Finally, the policy is updated using these decoupled signals, and high-value trajectories are consolidated back into memory as reusable skills.

In the following, we first introduce the core of our framework—the policy optimization algorithm (PARPO, \S\ref{sec:parpo}). We then detail the two supporting modules that make this optimization possible: the personalized reward model that provides user-conditioned signals (\S\ref{sec:reward}), and the skill graph memory that structures the rollout context (\S\ref{sec:psgm}).


\begin{figure}[t]
    \centering
    \includegraphics[width=\linewidth]{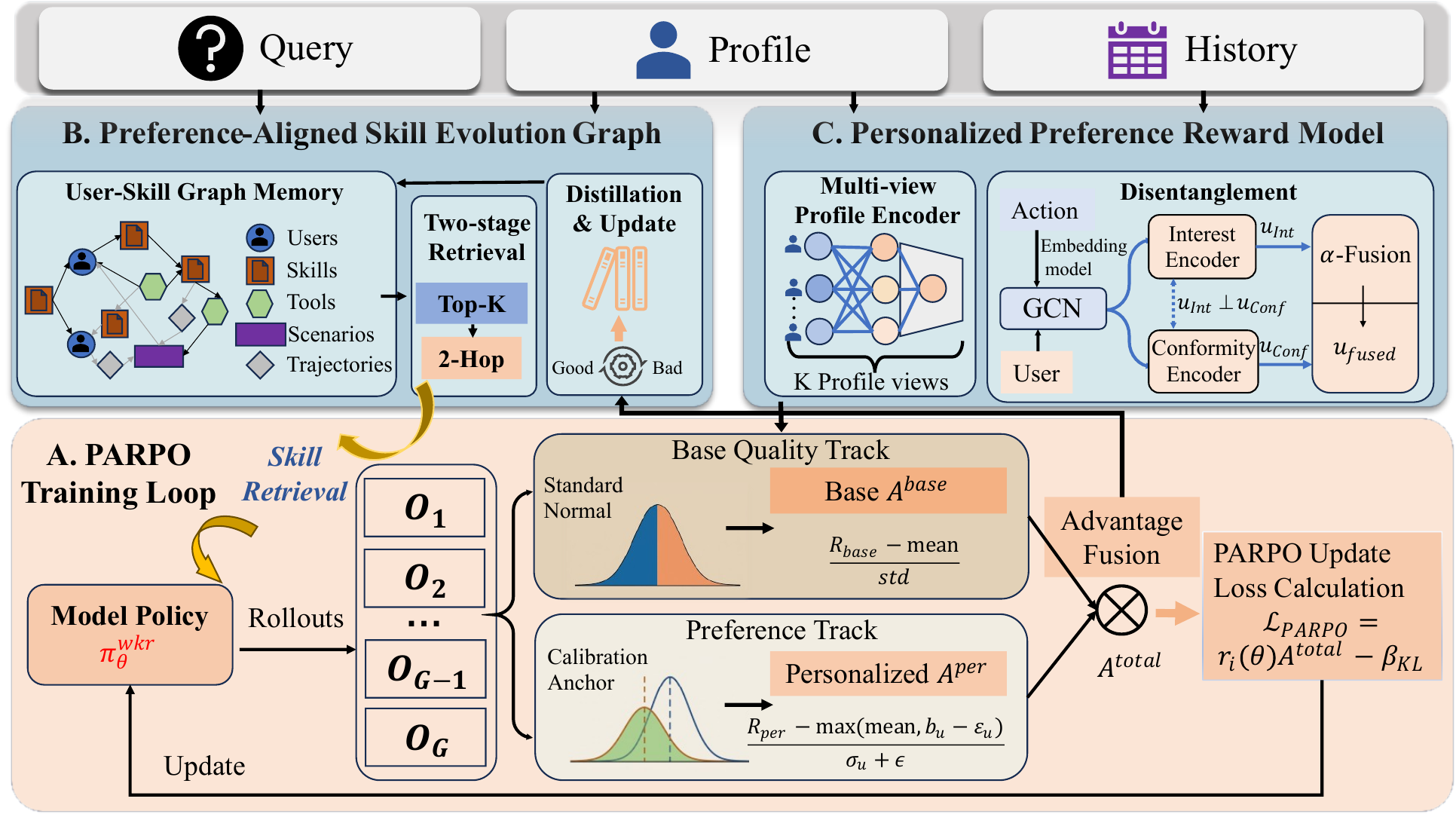}
    \caption{
    Overview of the proposed personalized Agentic RL framework. 
    }
    \label{fig:framework}
\end{figure}
\label{sec:framework}

\subsection{PARPO: Personalized Anchor Reward-Decoupled Policy Optimization}
\label{sec:parpo}

We now introduce the policy optimizer in our framework. PARPO separates task quality from personalized alignment, allowing the policy to improve user-specific behavior without entangling it with shared supervision. In implementation, PARPO is instantiated as a dual-track GRPO-style advantage estimator built on the personalized reward model in Section~\ref{sec:reward} and the graph memory in Section~\ref{sec:psgm}.

\paragraph{Theoretical justification.}
Appendix~\ref{app:theory} shows that under heterogeneous user preferences, personalized optimization is preferable to user-agnostic optimization, while standard GRPO incurs structural bias from pooled baselines and normalization. The analysis yields three conclusions. 
First, under heterogeneous user preferences, user-aware optimization is never worse than user-agnostic optimization: when different users prefer different trajectories for the same query, a single average-optimal policy necessarily compromises across users. Second, standard GRPO introduces structural bias in personalized settings because it uses pooled baselines and pooled normalization statistics. As a result, its advantage estimate can deviate from the true user-specific advantage due to both reward-center mismatch and reward-scale mismatch, with the dominant error controlled by cross-user preference heterogeneity. Third, PARPO reduces this bias by decoupling generic task-quality rewards from personalized preference rewards and calibrating the personalized branch with user-specific anchors.

In our implementation and experiments, PARPO reduces this bias primarily through reward decomposition and user-specific anchor calibration. In particular, for a fixed user $u$, its personalized advantage estimation error satisfies
\begin{equation}\small
\left|
\bar A_{\mathrm{PARPO}}(\tau\mid u,q)-\bar A_{\mathrm{pers}}^*(\tau\mid u,q)
\right|
\le
\frac{\delta_u+\epsilon_u}{\sigma_u(q)+\epsilon},
\label{eq:parpo_anchor_only_bound}
\end{equation}
where $\delta_u$ measures the estimation error of the user-specific historical anchor and $\epsilon_u$ is a conservative margin term. In expectation over users, this yields
\begin{equation}\small
\mathbb{E}_u\!\left[\left|\bar A_{\mathrm{PARPO}}(\tau\mid u,q)-\bar A_{\mathrm{pers}}^*(\tau\mid u,q)\right|\right]
\le
\frac{\bar\delta+\bar\epsilon}{\sigma_{\min}+\epsilon}.
\label{eq:parpo_anchor_only_bound_expectation}
\end{equation}
This shows that the practical benefit of PARPO in our setting comes from individual-specific baseline calibration: when a user's historical anchor is a better approximation to that user's true preference center than the pooled baseline, PARPO yields a tighter personalized advantage estimate for that user. A more general extension that also incorporates local grouping is deferred to Appendix~\ref{app:theory}.



Guided by this analysis, PARPO uses two explicit optimization tracks: a base track for generic task quality and a personalized track for user-contingent preference improvement.

\paragraph{Base advantage.}
For a sampled group of trajectories $\{\tau_i\}_{i=1}^G$ under the same prompt group, the base advantage follows standard within-group relative normalization:
\begin{equation}\small
A^{\mathrm{base}}_i =
\frac{R_{\mathrm{base}}(\tau_i)-\bar R_{\mathrm{base}}^{(g)}}
{\mathrm{Std}(\{R_{\mathrm{base}}(\tau_j)\}_{j\in g})+\epsilon}.
\label{eq:parpo_base_adv}
\end{equation}

\paragraph{Personalized advantage with user-anchor calibration.}
To stabilize optimization across heterogeneous users, PARPO maintains a persistent user-specific anchor for personalized rewards. For user $u$, let $m_u^{(t)}$ and $v_u^{(t)}$ denote the running mean and variance of personalized rewards. Given the current batch personalized rewards for that user, the anchor is updated by exponential moving average:
\begin{equation}\small
m_u^{(t+1)}=\rho m_u^{(t)} + (1-\rho)\bar R_{\mathrm{pers}}^{u,t}, \qquad
v_u^{(t+1)}=\rho v_u^{(t)} + (1-\rho)\mathrm{Var}(R_{\mathrm{pers}}^{u,t}).
\label{eq:parpo_anchor_update}
\end{equation}
The personalized branch then uses the following user-aware baseline:
\begin{equation}\small
b_{u,g}=\max\!\left(\bar R_{\mathrm{pers}}^{(g)},\; m_u-\gamma_p\sqrt{v_u}\right),
\label{eq:parpo_user_baseline}
\end{equation}
which prevents the personalized baseline from drifting too far above the user’s historical personalized reward center and provides a stable individual-specific calibration signal under heterogeneous reward scales. The personalized advantage is defined as
\begin{equation}\small
A^{\mathrm{pers}}_i=
\frac{R_{\mathrm{pers}}(\tau_i)-b_{u_i,g}}
{\sqrt{v_{u_i}}+\epsilon}.
\label{eq:parpo_pers_adv}
\end{equation}

\paragraph{Advantage fusion and policy update.}
The trajectory advantage is the weighted sum of the two:
\begin{equation}\small
A^{\mathrm{total}}_i
=
w_{\mathrm{base}}A^{\mathrm{base}}_i
+
w_{\mathrm{pers}}A^{\mathrm{pers}}_i.
\label{eq:parpo_total_adv}
\end{equation}
This fused advantage is then broadcast to the token level and used in a standard PPO-style clipped policy objective:
\begin{equation}\small
\mathcal{L}_{\mathrm{PARPO}}
=
\frac{1}{B}\sum_i
\max\!\left(
-r_i(\theta)A^{\mathrm{total}}_i,\;
-\mathrm{clip}(r_i(\theta),1-\eta,1+\eta)A^{\mathrm{total}}_i
\right),
\label{eq:parpo_loss}
\end{equation}
where $r_i(\theta)$ is the token-level policy ratio and $\eta$ the clipping coefficient. KL regularization, when enabled, is handled separately in the actor update loop rather than absorbed into the advantage.

In this way, PARPO explicitly separates objective task quality from user-contingent preference improvement, while using user-specific running statistics to mitigate cross-user reward-scale mismatch and stabilize personalized policy learning through individual-specific baseline calibration.

\subsection{Personalized Preference Reward Model with Two-Stage Preference Disentanglement}
\label{sec:reward}

As introduced in Section 4.1, PARPO requires a reliable personalized preference reward $R_{\mathrm{pers}}$. To provide this signal, we build a personalized preference reward model that provides user-conditioned neural preference signals for policy optimization through learned user representations and action compatibility scores.

\paragraph{Stage 1: Multi-view profile representation learning.}

To alleviate cold-start issues, we construct a profile representation from multiple semantic views. Given profile views $\{x_u^{(k)}\}_{k=1}^{K}$ for user $u$, we encode each view as $\mathbf{h}_u^{(k)} = E(x_u^{(k)})$ and compute attention weights
\begin{equation}
\small
\alpha_u^{(k)}=
\frac{
\exp\!\left(\mathbf{w}^{\top}\tanh(\mathbf{W}_{\mathrm{attn}}\mathbf{h}_u^{(k)}+\mathbf{b}_{\mathrm{attn}})\right)
}{
\sum_{k'=1}^{K}
\exp\!\left(\mathbf{w}^{\top}\tanh(\mathbf{W}_{\mathrm{attn}}\mathbf{h}_u^{(k')}+\mathbf{b}_{\mathrm{attn}})\right)
}.
\end{equation}
The fused profile representation is
\begin{equation}
\small
\mathbf{u}_{\mathrm{profile}}
=
\mathrm{LayerNorm}\!\left(
\mathbf{W}_{\mathrm{out}}
\sum_{k=1}^{K}\alpha_u^{(k)}\mathbf{h}_u^{(k)}
\right).
\end{equation}
To preserve view-specific information, we reconstruct each view embedding by $\hat{\mathbf{h}}_u^{(k)} = \mathbf{W}_{\mathrm{rec}}^{(k)}\mathbf{u}_{\mathrm{profile}}+\mathbf{b}_{\mathrm{rec}}^{(k)}$ and minimize
\begin{equation}
\small
\mathcal{L}_{\mathrm{recon}}
=
\sum_{u}\sum_{k=1}^{K}
\left\|
\hat{\mathbf{h}}_u^{(k)}-\mathbf{h}_u^{(k)}
\right\|_2^2.
\end{equation}

\paragraph{Stage 2: Collaborative preference disentanglement.}
We further incorporate collaborative preference signals from the user--item interaction graph. Specifically, LightGCN propagates embeddings as $\mathbf{E}^{(\ell+1)}=\hat{\mathbf{A}}\mathbf{E}^{(\ell)},$
and obtains the final collaborative representation by layer-wise averaging: $\mathbf{E}_{\mathrm{final}}=\frac{1}{L+1}\sum_{\ell=0}^{L}\mathbf{E}^{(\ell)}.$
Given the collaborative user representation $\mathbf{u}_{\mathrm{cf}}$, we learn two branches to capture interest and conformity signals:$
\mathbf{u}_{\mathrm{int}}=\mathrm{InterestEncoder}(\mathbf{u}_{\mathrm{cf}}),
\mathbf{u}_{\mathrm{conf}}=\mathrm{ConformityEncoder}(\mathbf{u}_{\mathrm{cf}}).
$
The two normalized branch embeddings are then fused as
$
\mathbf{u}_{\mathrm{fused}}=s\!\left(\alpha_{\mathrm{int}}\hat{\mathbf{u}}_{\mathrm{int}}+\alpha_{\mathrm{conf}}\hat{\mathbf{u}}_{\mathrm{conf}}\right),
$
where $\hat{\mathbf{u}}_{\mathrm{int}}$ and $\hat{\mathbf{u}}_{\mathrm{conf}}$ denote the normalized interest and conformity embeddings, respectively.

This branch structure alone does \emph{not} guarantee true separation of intrinsic preference from popularity-, conformity-, or group-level bias. We therefore treat disentanglement operationally and impose it through branch-specific objectives. The interest branch upweights less popular items:
\begin{equation}
\footnotesize
\mathcal{L}_{\mathrm{int}}
=
\frac{1}{B}\sum_{(u,i^{+})}
\left[
-\log\!\left(\omega_{i^{+}}^{\mathrm{int}}+\epsilon\right)
-
\frac{
\mathbf{u}_{\mathrm{int}}^{\top}\mathbf{i}_{\mathrm{cf}}^{+}
}{\tau}
+
\log \sum_{j}
\exp\!\left(
\frac{
\mathbf{u}_{\mathrm{int}}^{\top}\mathbf{i}_{\mathrm{cf}}^{(j)}
}{\tau}
\right)
\right],
\qquad
\omega_{i}^{\mathrm{int}}=\exp(1-\tilde p_i).
\end{equation}
Here $\tilde p_i\in[0,1]$ denotes normalized popularity. The conformity branch uses the same objective with opposite weighting, i.e., $\omega_i^{\mathrm{conf}}=\exp(\tilde p_i)$. We further regularize both branches by

\begin{equation}
\small
\mathcal{L}_{\mathrm{orth}}
=
\frac{1}{B}
\sum_u
\left(
\hat{\mathbf{u}}_{\mathrm{int}}^{\top}\hat{\mathbf{u}}_{\mathrm{conf}}
\right)^2,
\end{equation}
which discourages overlap but is not, by itself, evidence of full causal disentanglement.

For an action text $a$, we encode it as $\mathbf{a}_{\mathrm{proj}} = \mathrm{ActionEncoder}(E(a))$ and compute the personalized score by $r_{\mathrm{fused}}(u,a) = \mathbf{u}_{\mathrm{fused}}^{\top}\mathbf{a}_{\mathrm{proj}}$. The resulting neural preference score combines collaborative structure, branch-specific supervision, and regularization, and is later calibrated and integrated with LLM-based evaluation in the environment-level reward pipeline.
\subsection{Preference-Aligned Skill Evolution Graph Memory}
\label{sec:psgm}
Beyond the reward signal, personalized policy learning also requires a structured behavioral context during rollout. To this end, we maintain a heterogeneous graph memory $\mathcal{G}=(\mathcal{V},\mathcal{E})$ over users, skills, tools, scenarios, and trajectories, with typed edges encoding ownership, applicability, complementarity, conflict, execution history, and scenario triggers.

Rather than using a flat retrieval index, PSGM organizes skills with both node embeddings and graph structure. Skill and user semantic embeddings are stored on nodes, while graph connectivity provides signals such as community membership, complementary neighbors, and conflicting neighbors. To capture multi-granularity structure, we perform hierarchical community detection with Leiden when available and Louvain otherwise. At each level, communities are obtained by maximizing modularity:
\begin{equation}\small
Q=\frac{1}{2m}\sum_{i,j}
\left(A_{ij}-\frac{k_i k_j}{2m}\right)\delta(c_i,c_j).
\end{equation}

At inference, retrieval has two stages. We first retrieve semantic candidates from the skill set:
\begin{equation}\small
\mathcal{S}_{\mathrm{init}}(q)
=
\operatorname{TopM}_{s\in\mathcal{S}}
\operatorname{sim}(q,s).
\end{equation}
We then expand the candidates with a 2-hop traversal: from each retrieved skill to its owner user, and then to sibling skills of that user, injecting personalized local structure into the pool.

Each candidate is ranked by the following graph-aware score:
\begin{equation}\small
\operatorname{score}(q,s,p_u)
=
f_{\mathrm{sem}}(q,s)\,
\bigl(\alpha+\beta f_{\mathrm{user}}(p_u,s)\bigr)\,
\bigl(1+\gamma f_{\mathrm{comm}}(p_u,s)\bigr)\,
f_{\mathrm{comp}}(s)\,
\bigl(1-\delta f_{\mathrm{conf}}(s)\bigr).
\end{equation}
Here, $f_{\mathrm{sem}}$ is query--skill similarity, $f_{\mathrm{user}}$ user--skill similarity, $f_{\mathrm{comm}}$ community relevance, $f_{\mathrm{comp}}$ complement boost, and $f_{\mathrm{conf}}$ a conflict penalty; $\alpha,\beta,\gamma,\delta$ are fixed graph-level hyperparameters.

The top-ranked skills are inserted into the rollout context as structured personalized memory. Thus, PSGM improves rollout-time personalization by exposing preference-relevant and graph-consistent skills before each decision step, without changing policy parameters.

\section{Experiments}
\label{sec:5}
\subsection{Experimental Setup}
\subsubsection{Baselines}
We compare against prompting, memory, RL, and memory-RL baselines, including \textsc{ReAct}, \textsc{Mem0}, \textsc{GRPO}, \textsc{DAPO}, \textsc{GSPO}, \textsc{GiGPO}, \textsc{MemRL}, and \textsc{SkillRL}. We also report GPT-4o and Claude Sonnet 4 as closed-source references. All open-source methods are evaluated under the same model scales, task settings, and tool interfaces whenever applicable. We further include a GDPO-style variant in the ablation study, which removes PARPO’s user-anchor calibration while keeping the same personalized reward and memory components. For baselines in personalized environments, we additionally provide the native personalization reward to ensure fair comparison.

\subsubsection{Benchmarks, Evaluation, and Training Details}
We evaluate on ETAPP and SJAgent. ETAPP is a public benchmark for personal assistant agents, covering user behaviors in daily-life scenarios; we further construct a more challenging split, ETAPP-Hard. SJAgent is a realistic environment for merchant decision-making and recommendation, built from merchant data on a major Chinese e-commerce platform. Environment details are provided in Appendix~\ref{app:env}, and the construction pipeline of ETAPP-Hard is given in Appendix~\ref{app:etapp-hard}.

For ETAPP, we report Judge, Personal, Proactive, and Procedure. For SJAgent, we report Reward, Data Authenticity, Business Logic, Merchant Profile Match, Task Completion, and Market Analysis Depth. Metric definitions, official evaluation prompts, and scoring criteria are provided in Appendix~\ref{app:eval}, and training hyperparameters are listed in Appendix~\ref{app:train}.
\subsection{Main Results on Personalized Decision-Making Benchmarks}
\begin{table*}[t]
  \centering
  \scriptsize
  \caption{Main Results on ETAPP, ETAPP-Hard and SJAgent. Ours significantly outperforms the strongest baseline, SkillRL, under the paired t-test (p < 0.005)}
  \label{tab:merged-results-final-structured}
  \setlength{\tabcolsep}{3pt}
  \renewcommand{\arraystretch}{1.0}
  \begin{tabular}{lccccccccccc}
    \toprule
    \multirow{2}{*}{\textbf{Metric}} & \multicolumn{2}{c}{\textbf{Closed-source}} & \multicolumn{9}{c}{\textbf{Open-source Models}} \\
    \cmidrule(lr){2-3} \cmidrule(lr){4-12}
    & \textbf{GPT-4o} & \textbf{Claude S4} & \textbf{ReAct} & \textbf{Mem0} & \textbf{GRPO} & \textbf{DAPO} & \textbf{GSPO} & \textbf{GiGPO} & \textbf{MemRL} & \textbf{SkillRL} & \textbf{Ours} \\
    \midrule

    \multicolumn{12}{l}{\textbf{Scale: Qwen3-4B Models}} \\
    \midrule
    \textit{ETAPP-Original} & & & & & & & & & & & \\
    \quad Personal & 2.5882 & 3.9705 & 2.3765 & 2.5294 & 3.5000 & 3.7406 & 3.8094 & 3.9375 & 3.1832 & 4.0312 & \textbf{4.2344} \\
    \quad Proactive & 2.1529 & 3.1871 & 1.5176 & 1.7882 & 3.1250 & 3.2125 & 3.4062 & 3.4156 & 2.3105 & 3.3662 & \textbf{3.4844} \\
    \quad Procedure & 2.6353 & 3.5365 & 3.4941 & 3.5882 & 3.4688 & 3.3125 & 3.2250 & 3.3906 & 2.9031 & 3.7188 & \textbf{3.8438} \\
    \quad \textbf{Judge} & 0.4918 & 0.7129 & 0.4925 & 0.5270 & 0.6729 & 0.6844 & 0.6960 & 0.7162 & 0.5598 & 0.7411 & \textbf{0.7708} \\
    \addlinespace[2pt]
    \textit{ETAPP-Hard} & & & & & & & & & & & \\
    \quad Personal & 2.1560 & 3.6081 & 1.6156 & 2.6823 & 3.4875 & 3.6250 & 3.7187 & 3.7188 & 3.2875 & 3.8656 & \textbf{4.0469} \\
    \quad Proactive & 2.3449 & 3.5550 & 0.8688 & 1.8470 & 2.8750 & 2.9719 & 3.2938 & 3.3375 & 2.6344 & 3.2406 & \textbf{3.3313} \\
    \quad Procedure & 2.1750 & 3.0488 & 1.9156 & 3.8588 & 3.1031 & 3.4156 & 3.0500 & 3.0531 & 2.8281 & 3.4094 & \textbf{3.5344} \\
    \quad \textbf{Judge} & 0.4451 & 0.6808 & 0.2933 & 0.4318 & 0.6310 & 0.6675 & 0.6708 & 0.6740 & 0.5833 & 0.7010 & \textbf{0.7275} \\
    \addlinespace[2pt]
    \textit{SJAgent} & & & & & & & & & & & \\
    \quad Data Auth. & 1.0000 & 1.1350 & 1.8540 & 2.1450 & 2.6142 & 2.7480 & 2.9412 & 2.9550 & 2.2150 & 3.0523 & \textbf{3.3790} \\
    \quad Business Logic & 2.9170 & 2.8290 & 1.6215 & 1.9520 & 2.3085 & 2.6515 & 2.7210 & 2.7845 & 2.0450 & 2.8703 & \textbf{3.0830} \\
    \quad Profile Match & 3.8860 & 3.7920 & 1.9840 & 2.2140 & 2.5241 & 2.8845 & 2.8640 & 2.8941 & 2.3120 & 2.8543 & \textbf{2.8960} \\
    \quad Task Compl. & 3.9030 & 3.2610 & 2.1450 & 2.4580 & 2.8120 & 3.1415 & 3.1613 & 3.1842 & 2.5480 & 3.2442 & \textbf{3.4150} \\
    \quad Market Depth & 2.8060 & 2.7080 & 2.0155 & 2.3110 & 2.9292 & 2.9025 & 3.2345 & 3.2442 & 2.4200 & 3.2269 & \textbf{3.4450} \\
    \quad \textbf{Reward $\uparrow$} & 0.7256 & 0.6862 & 0.4810 & 0.5540 & 0.6594 & 0.7164 & 0.7461 & 0.7531 & 0.5770 & 0.7625 & \textbf{0.8109} \\

    \midrule
    \multicolumn{12}{l}{\textbf{Scale: Qwen3-8B Models}} \\
    \midrule
    \textit{ETAPP-Original} & & & & & & & & & & & \\
    \quad Personal & 2.5882 & 3.9705 & 2.1529 & 2.6823 & 3.9531 & 4.0000 & 4.0469 & 3.9375 & 3.1406 & 4.0938 & \textbf{4.2344} \\
    \quad Proactive & 2.1529 & 3.1871 & 1.4706 & 1.8470 & 3.5000 & 3.5094 & 3.4531 & 3.5625 & 2.7656 & 3.7656 & \textbf{4.0938} \\
    \quad Procedure & 2.6353 & 3.5365 & 3.4000 & 3.8588 & 3.7344 & 3.7469 & 3.9688 & 3.9375 & 2.9062 & 3.8281 & \textbf{4.1719} \\
    \quad \textbf{Judge} & 0.4918 & 0.7129 & 0.4682 & 0.5592 & 0.7458 & 0.7504 & 0.7646 & 0.7625 & 0.5875 & 0.7792 & \textbf{0.8333} \\
    \addlinespace[2pt]
    \textit{ETAPP-Hard} & & & & & & & & & & & \\
    \quad Personal & 2.1560 & 3.6081 & 1.7187 & 2.3312 & 3.8193 & 3.9219 & 3.9625 & 3.9531 & 2.9969 & 4.1188 & \textbf{4.3187} \\
    \quad Proactive & 2.3449 & 3.5550 & 0.9000 & 1.4656 & 3.2854 & 3.2938 & 3.3031 & 3.2781 & 2.2860 & 3.4094 & \textbf{3.7844} \\
    \quad Procedure & 2.1750 & 3.0488 & 1.8813 & 2.6812 & 3.3721 & 3.4625 & 3.5156 & 3.5688 & 2.5578 & 3.4625 & \textbf{3.8719} \\
    \quad \textbf{Judge} & 0.4451 & 0.6808 & 0.3000 & 0.4318 & 0.6995 & 0.7119 & 0.7188 & 0.7200 & 0.5227 & 0.7327 & \textbf{0.7983} \\
    \addlinespace[2pt]
    \textit{SJAgent} & & & & & & & & & & & \\
    \quad Data Auth. & 1.0000 & 1.1350 & 1.9420 & 2.3120 & 2.9105 & 3.0540 & 3.0410 & 3.1250 & 2.4580 & 3.1016 & \textbf{3.4120} \\
    \quad Business Logic & 2.9170 & 2.8290 & 1.7850 & 2.0450 & 2.6450 & 2.9120 & 2.6845 & 2.9410 & 2.1850 & 2.9884 & \textbf{3.1080} \\
    \quad Profile Match & 3.8860 & 3.7920 & 2.0120 & 2.3580 & 2.8142 & 3.1050 & 2.7425 & 3.1020 & 2.5140 & 3.0243 & \textbf{3.1415} \\
    \quad Task Compl. & 3.9030 & 3.2610 & 2.3140 & 2.6140 & 3.0453 & 3.2840 & 3.1040 & 3.3150 & 2.7840 & 3.3687 & \textbf{3.4240} \\
    \quad Market Depth & 2.8060 & 2.7080 & 2.2470 & 2.4710 & 3.0850 & 3.2450 & 3.3480 & 3.3570 & 2.6590 & 3.3670 & \textbf{3.4545} \\
    \quad \textbf{Reward $\uparrow$} & 0.7256 & 0.6862 & 0.5150 & 0.5900 & 0.7250 & 0.7800 & 0.7461 & 0.7920 & 0.6300 & 0.7925 & \textbf{0.8270} \\
    \bottomrule
  \end{tabular}
\end{table*}

Table~\ref{tab:merged-results-final-structured} reports results on ETAPP, ETAPP-Hard, and SJAgent. Our method achieves the best overall performance under both 4B and 8B settings and outperforms all open-source baselines on most metrics. On ETAPP and ETAPP-Hard, it obtains the highest judge scores, with larger gains on ETAPP-Hard, indicating stronger robustness in challenging personalized scenarios. It also performs best on Personal, Proactive, and Procedure, showing better preference alignment and decision quality. On SJAgent, our method again achieves the highest Reward and leads on key dimensions including Data Authenticity, Business Logic, Task Completion, and Market Analysis Depth, demonstrating strong cross-domain generalization.

\subsection{Ablation Study}
\begin{table*}[t]
\centering
\small
\setlength{\tabcolsep}{6pt}
\renewcommand{\arraystretch}{1.0}
\caption{Ablation results on ETAPP. 
For each variant, we report the judge score, and list the corresponding \textbf{personal}, \textbf{proactive}, \textbf{procedure}, and \textbf{judge} scores. 
$\Delta$ denotes the absolute difference in \textbf{judge\_best} compared with the full model. }
\label{tab:ablation_full_metrics}
\begin{tabular}{llccccc}
\toprule
\textbf{Group} & \textbf{Variant} & \textbf{Personal $\uparrow$} & \textbf{Proactive $\uparrow$} & \textbf{Procedure $\uparrow$} & \textbf{Judge $\uparrow$} & \textbf{$\Delta$ Judge} \\
\midrule
-- & Full (baseline)         & \textbf{4.2344} & 3.4844 & \textbf{3.8438} & \textbf{0.7708} & 0.0000 \\
\midrule
A & A1 no\_memory      & 3.6562 & 3.1938 & 3.6594 & 0.7006 & -0.0702 \\
A & A2 skill\_original             & 4.1562 & 3.4188 & 3.7625 & 0.7558 & -0.0150 \\
\midrule
B & B1 no\_community               & 4.1625 & 3.4031 & 3.7906 & 0.7571 & -0.0137 \\
B & B2 louvain                     & 4.0625 & \textbf{3.6719} & 3.7656 & 0.7667 & -0.0041 \\
B & B3 no\_dynupd                  & 3.9688 & 3.3438 & 3.7969 & 0.7406 & -0.0302 \\
B & B4 bm25                        & 4.1094 & 3.4656 & 3.8125 & 0.7592 & -0.0116 \\
\midrule
C & C1 grpo\_only                  & 3.9250 & 3.3625 & 3.5500 & 0.7225 & -0.0483 \\
C & C2 no\_base                    & 3.7656 & 3.3594 & 3.6719 & 0.7198 & -0.0510 \\
C & C3 no\_pers                    & 3.7500 & 3.3281 & 3.7344 & 0.7208 & -0.0500 \\
C & C4 GDPO-style (w/o anchor)     & 3.9033 & 3.3694 & 3.5971 & 0.7247 & -0.0461 \\
C & C5 no\_rm\_model               & 4.0000 & 3.5312 & 3.8750 & 0.7604 & -0.0104 \\
C & C6 no\_interest                & 3.8464 & 3.2986 & 3.6372 & 0.7188 & -0.0588 \\
C & C7 no\_conformity              & 3.7215 & 3.4108 & 3.6884 & 0.7214 & -0.0494 \\
\bottomrule
\end{tabular}
\end{table*}
Table~\ref{tab:ablation_full_metrics} reports ablation results on ETAPP. A1 removes skill memory, and A2 replaces PSGM with the original flat skill pipeline. Within PSGM, B1 removes hierarchical communities, B2 replaces Leiden with Louvain, B3 disables dynamic graph updates, and B4 replaces dense retrieval with BM25. Within reward learning, C1 replaces PARPO with GRPO, C2/C3 remove the base/personalized reward branch, C4 adopts a GDPO-style variant without user-anchor calibration, C5 disables the personalized neural reward model, and C6/C7 remove the interest/conformity branch.

All ablations reduce performance. Removing skill memory causes the largest Judge drop (0.7708 → 0.7006), showing that memory is central to personalized decision-making. Weakening PSGM also consistently hurts performance, validating structured retrieval and dynamic graph updates. In reward learning, all variants underperform PARPO; notably, the GDPO-style variant remains worse than PARPO, showing that reward decoupling alone is insufficient without user-anchor calibration.

\subsection{Rollout Evaluation with Human and LLM Judges on ETAPP}
To evaluate personalized behavior on user-conditioned tasks, we conduct a blinded rollout study on 20 ETAPP tasks with rich user-specific contexts. For each task, PARPO, SkillRL, Claude 3.5 Sonnet, and GPT-4o independently generate responses, which are then anonymized and scored by 15 human experts and 4 LLM judges (GPT-5.4, GLM-5.1, Kimi-K2.6, and Gemini-3.1-Pro) along \textit{Question Relevance}, \textit{User Relevance}, and \textit{Readability}. 

Figures~\ref{fig:main_eval_a} and \ref{fig:main_eval_b} show that PARPO achieves the highest average score under both human and LLM evaluation, and all five judge groups rank it above all baselines. The largest gain appears on \textit{User Relevance}, indicating stronger personalization rather than merely better fluency. Figure~\ref{fig:human_expert_agreement} further shows that the gains are consistent across individual human experts, indicating robust preference rather than isolated wins. A representative case further illustrating this behavior is shown in Appendix~\ref{app:case_shopping_cart}.

\begin{figure*}[!tb]
\centering
\captionsetup[subfigure]{justification=centering, singlelinecheck=false, skip=0pt}

\begin{subfigure}[t]{0.24\textwidth}
    \centering
    \includegraphics[width=\linewidth,height=6.0cm,keepaspectratio]{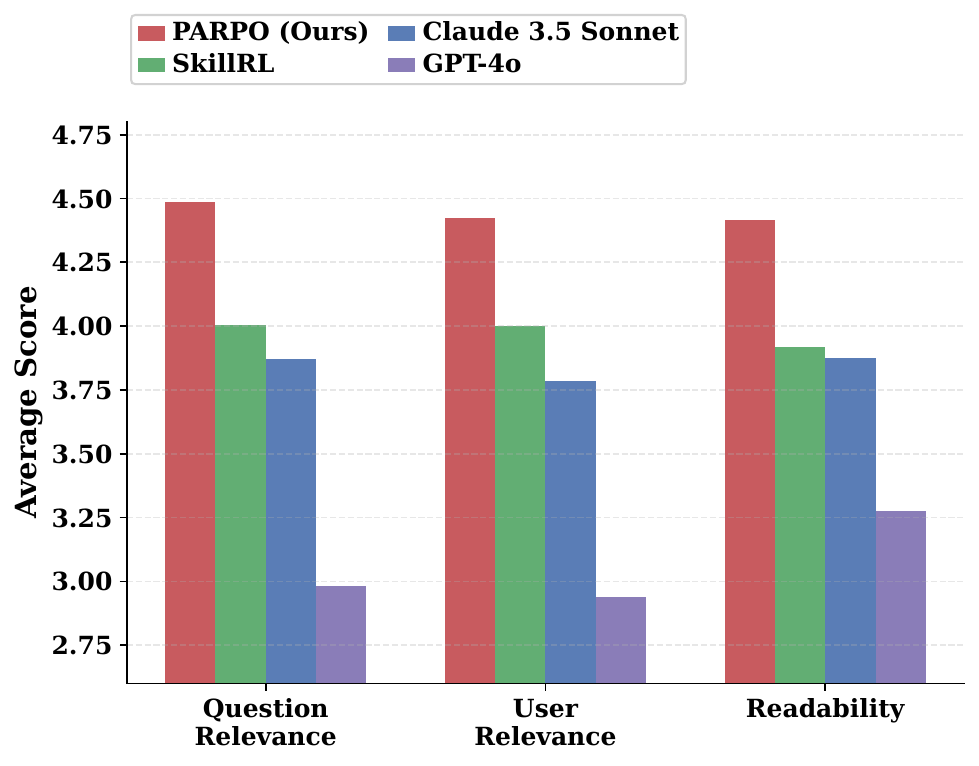}
    \caption{Human scores by dimension.}
    \label{fig:main_eval_a}
\end{subfigure}
\hfill
\begin{subfigure}[t]{0.24\textwidth}
    \centering
    \includegraphics[width=\linewidth,height=6.0cm,keepaspectratio]{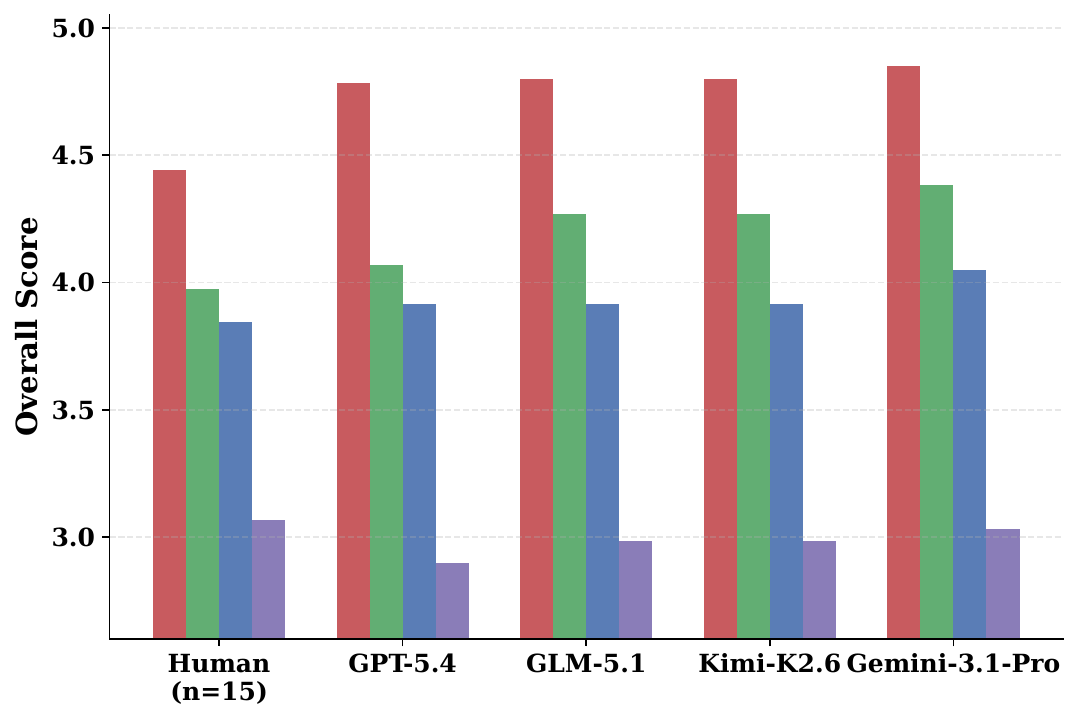}
    \caption{Overall scores from judges.}
    \label{fig:main_eval_b}
\end{subfigure}
\hfill
\begin{subfigure}[t]{0.50\textwidth}
    \centering
    \includegraphics[width=\linewidth,height=2.3cm,keepaspectratio]{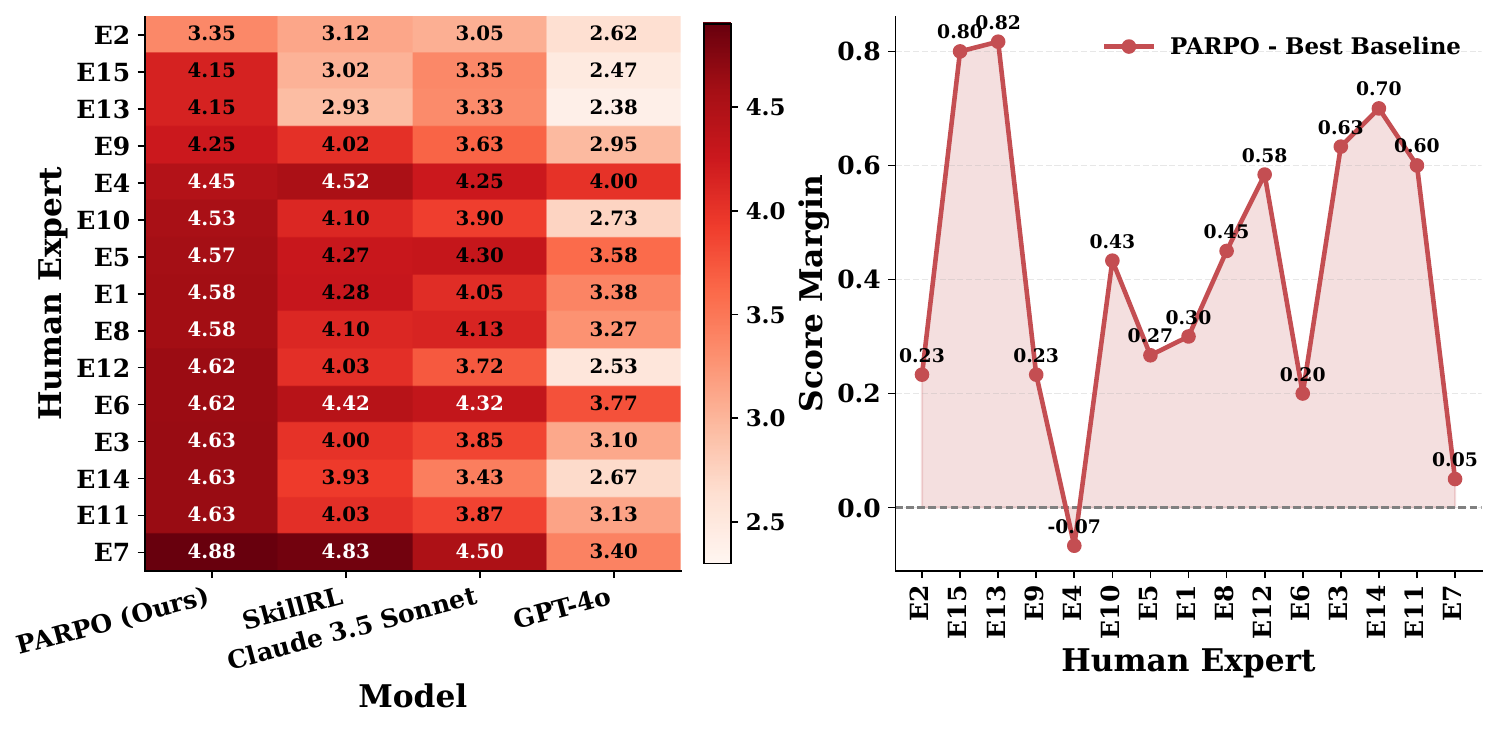}
    \caption{Expert agreement and consistency.}
    \label{fig:human_expert_agreement}
\end{subfigure}
\caption{Blinded evaluation on 20 personalized ETAPP tasks. Left: human scores by dimension. Middle: overall scores from human and LLM judges. Right: expert-level agreement, including the expert-by-model heatmap and PARPO's margin over the strongest baseline for each human expert.}
\label{fig:combined_eval}
\end{figure*}

\subsection{Training Dynamics and Skill Evolution Analysis}
\begin{figure*}[t]
    \centering

    \begin{subfigure}[t]{0.96\textwidth}
        \centering
        \includegraphics[width=\textwidth]{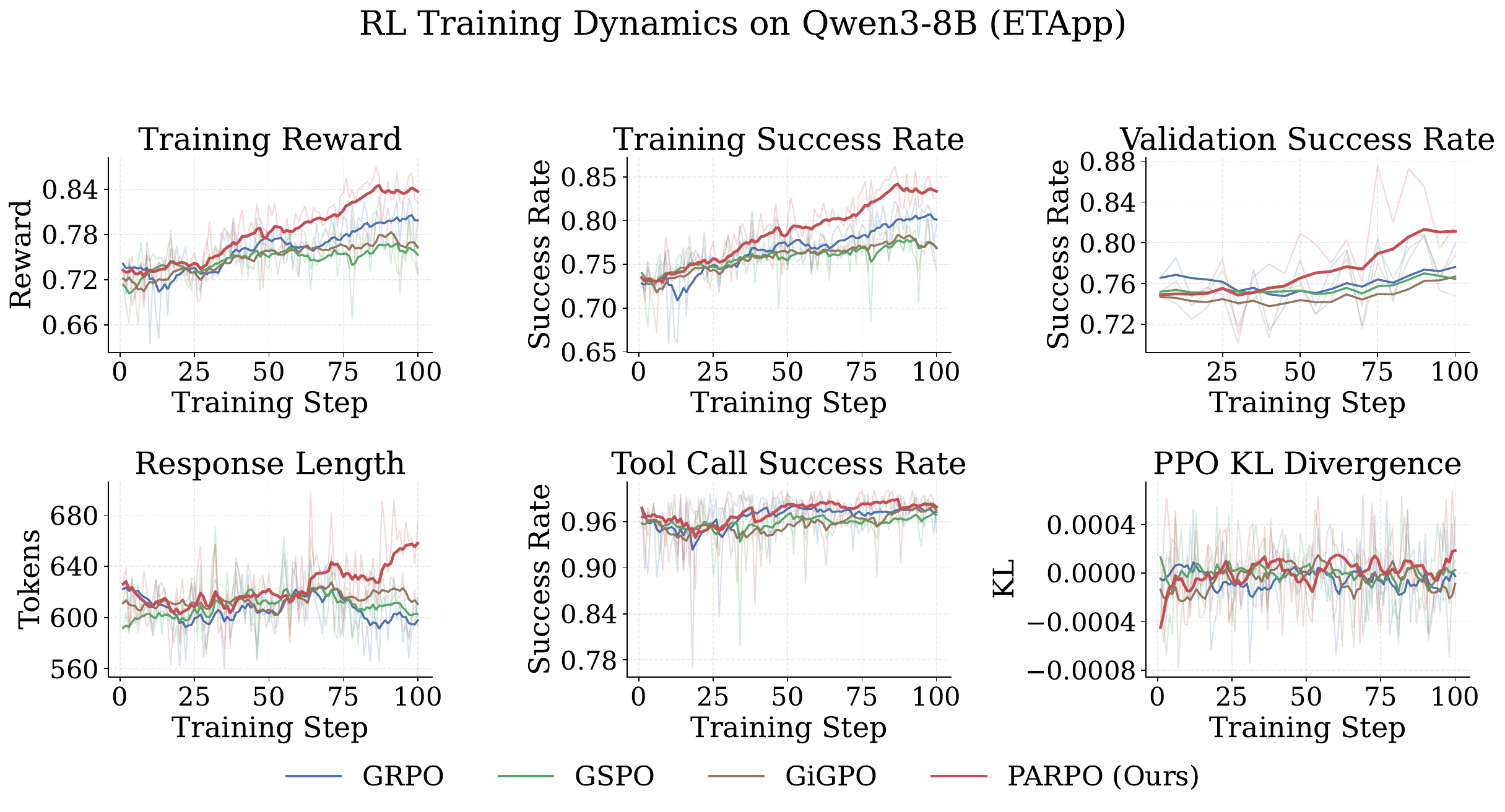}
        \caption{RL training dynamics on ETAPP.}
        \label{fig:rl_training_dynamics_etapp}
    \end{subfigure}


    \begin{subfigure}[t]{0.26\textwidth}
        \centering
        \includegraphics[width=\textwidth]{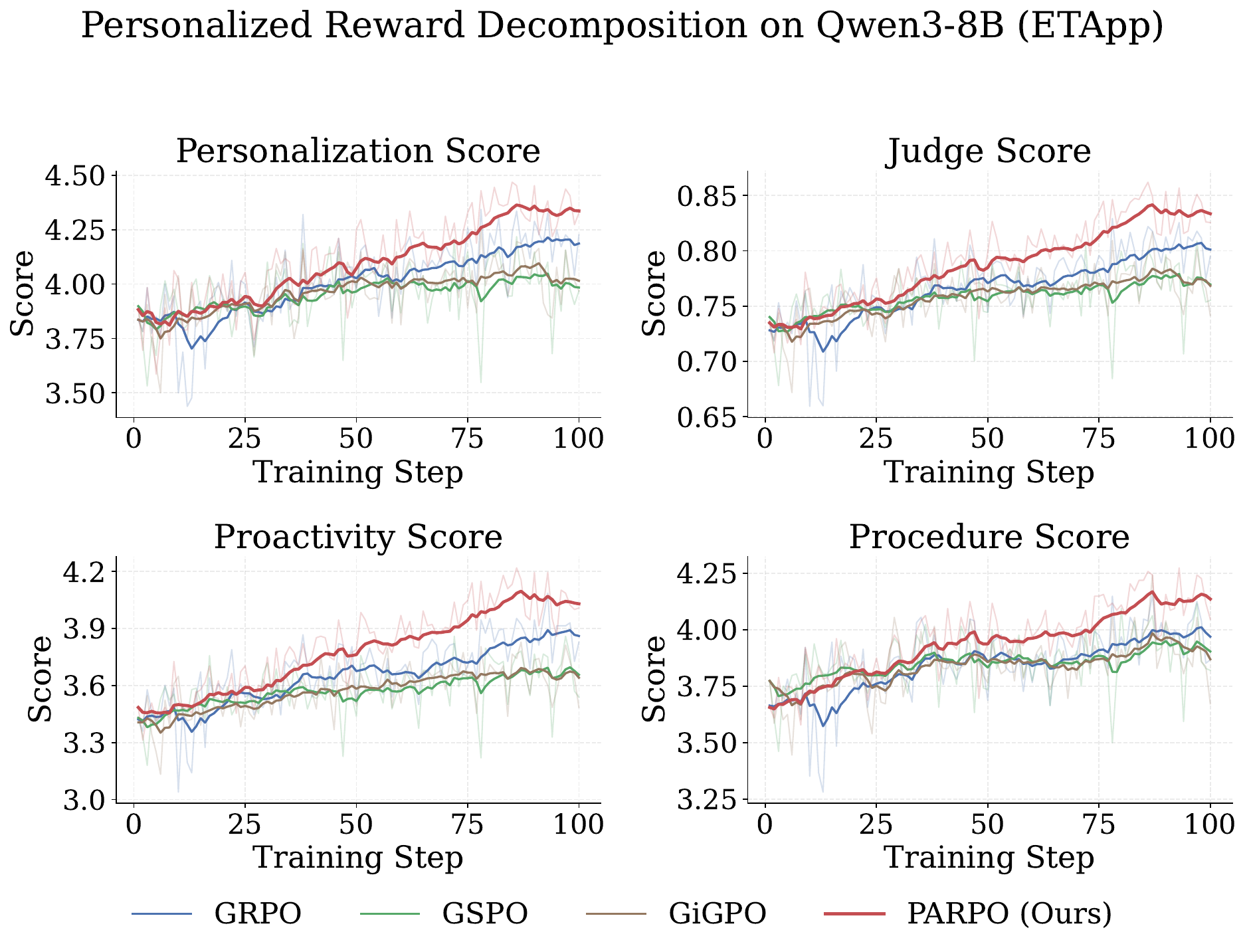}
        \caption{Personalized reward decomposition on ETAPP.}
        \label{fig:reward_decomposition_etapp}
    \end{subfigure}
    \hfill
    \begin{subfigure}[t]{0.41\textwidth}
        \centering
        \includegraphics[width=\textwidth]{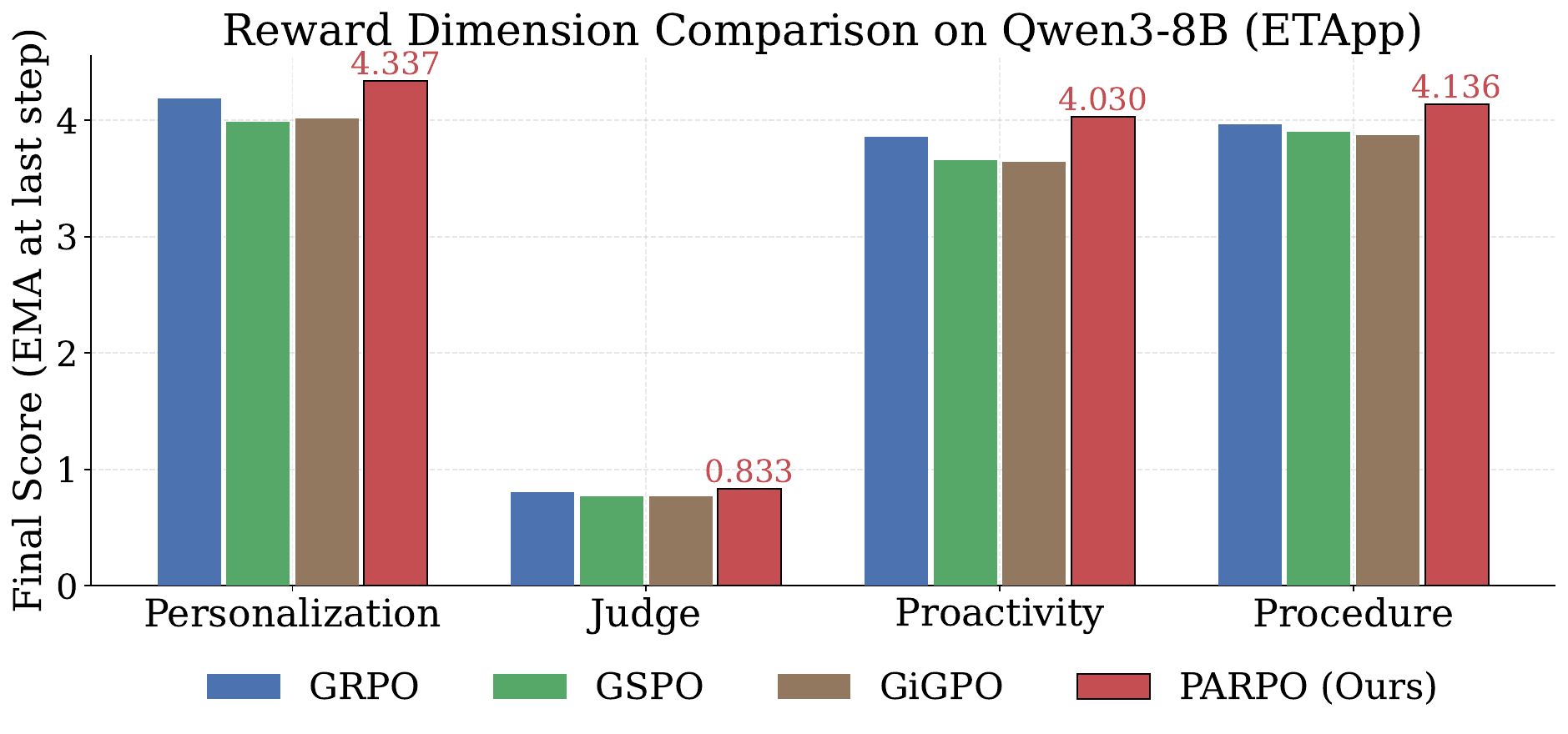}
        \caption{Final EMA scores across reward dimensions on ETAPP.}
        \label{fig:reward_dimension_comparison_etapp}
    \end{subfigure}
    \hfill
    \begin{subfigure}[t]{0.31\textwidth}
        \centering
        \includegraphics[width=\textwidth]{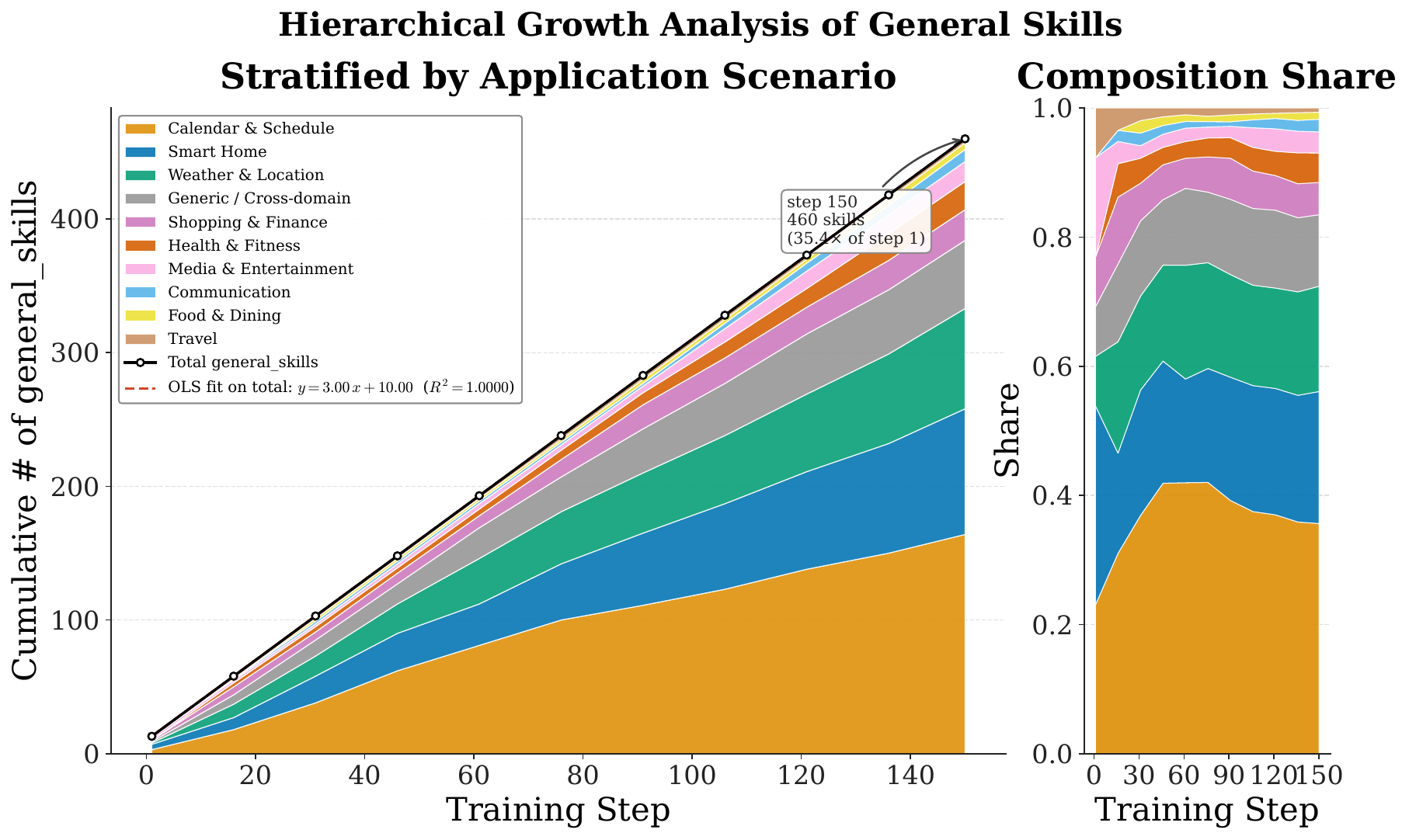}
        \caption{Hierarchical growth of general skills during training.}
        \label{fig:general_skill_growth}
    \end{subfigure}

    \caption{Training dynamics and skill evolution analysis of Qwen3-8B on ETAPP. Top: RL training dynamics. Bottom-left: reward decomposition during training. Bottom-middle: final EMA scores by reward dimension. Bottom-right: hierarchical growth of general skills during training. All experiments exclude the personalized reward model and Skill Graph.}
    \label{fig:rl_analysis_etapp}
\end{figure*}
All experiments in this section are conducted without the personalized reward model or Skill Graph to isolate the effect of RL optimization.
\subsubsection{Comparison of RL Training Strategies}
We compare GRPO, GSPO, GiGPO, and PARPO on ETAPP. Figure~\ref{fig:rl_training_dynamics_etapp} shows that PARPO yields the strongest training dynamics, with higher reward, better training/validation success rates, stronger tool-call success, and stable KL divergence. Similar trends hold on ETAPP-Hard (Appendix~\ref{app:hard_rl_training}).
\subsubsection{Effect of Personalized Reward Optimization and Skill Evolution}
We further analyze the four ETAPP reward dimensions: personalization, judge, proactivity, and procedure. Figures~\ref{fig:reward_decomposition_etapp} and \ref{fig:reward_dimension_comparison_etapp} show that PARPO improves all dimensions during training and achieves the best EMA scores, especially on personalization, indicating that explicit personalized reward optimization provides a stronger learning signal. Moreover, Figure~\ref{fig:general_skill_growth} shows that the cumulative number of general skills grows steadily with training steps, while the composition across application scenarios remains broadly stable. This suggests that training continuously consolidates reusable cross-domain skills rather than improving only a narrow subset of scenarios. ETAPP-Hard results are deferred to Appendix~\ref{app:hard_reward_optimization}.

\section{Conclusion and Limitations}
\label{sec:6}

This work studied personalized Agentic RL for user-conditioned tasks, where optimal behavior depends on user preferences rather than a single correct answer. We proposed a unified framework centered on PARPO, which decouples general-quality and personalized rewards and stabilizes training with user-specific anchors. Experiments on ETAPP, ETAPP-Hard, and SJAgent showed consistent and significant gains in personalization while preserving overall task quality, highlighting the value of training-time personalized optimization for user-centric agents.
A key limitation of this work is the scale of the human evaluation: due to annotation cost and limited expert availability, we reported judgments from only 15 experts on 20 randomly sampled examples. Although these results broadly agree with the LLM evaluations, future work should expand the human study to a larger and more diverse pool of annotators and evaluation instances to better assess robustness, generalizability, and real-world validity.



\newpage
\bibliographystyle{unsrtnat}
\bibliography{refs}

@inproceedings{
    gehring2025rlef,
    title={{RLEF}: Grounding Code {LLM}s in Execution Feedback with Reinforcement Learning},
    author={Jonas Gehring and Kunhao Zheng and Jade Copet and Vegard Mella and Taco Cohen and Gabriel Synnaeve},
    booktitle={Forty-second International Conference on Machine Learning},
    year={2025},
    url={https://openreview.net/forum?id=PzSG5nKe1q}
}

@article{le2022coderl,
  title={Coderl: Mastering code generation through pretrained models and deep reinforcement learning},
  author={Le, Hung and Wang, Yue and Gotmare, Akhilesh Deepak and Savarese, Silvio and Hoi, Steven Chu Hong},
  journal={Advances in Neural Information Processing Systems},
  volume={35},
  pages={21314--21328},
  year={2022}
}

@article{yang2024swe,
  title={Swe-agent: Agent-computer interfaces enable automated software engineering},
  author={Yang, John and Jimenez, Carlos E and Wettig, Alexander and Lieret, Kilian and Yao, Shunyu and Narasimhan, Karthik and Press, Ofir},
  journal={Advances in Neural Information Processing Systems},
  volume={37},
  pages={50528--50652},
  year={2024}
}

@article{nakano2021webgpt,
  title={Webgpt: Browser-assisted question-answering with human feedback},
  author={Nakano, Reiichiro and Hilton, Jacob and Balaji, Suchir and Wu, Jeff and Ouyang, Long and Kim, Christina and Hesse, Christopher and Jain, Shantanu and Kosaraju, Vineet and Saunders, William and others},
  journal={arXiv preprint arXiv:2112.09332},
  year={2021}
}

@inproceedings{
    qi2025webrl,
    title={Web{RL}: Training {LLM} Web Agents via Self-Evolving Online Curriculum Reinforcement Learning},
    author={Zehan Qi and Xiao Liu and Iat Long Iong and Hanyu Lai and Xueqiao Sun and Jiadai Sun and Xinyue Yang and Yu Yang and Shuntian Yao and Wei Xu and Jie Tang and Yuxiao Dong},
    booktitle={The Thirteenth International Conference on Learning Representations},
    year={2025},
    url={https://openreview.net/forum?id=oVKEAFjEqv}
}

@inproceedings{lai2024autowebglm,
  title={Autowebglm: A large language model-based web navigating agent},
  author={Lai, Hanyu and Liu, Xiao and Iong, Iat Long and Yao, Shuntian and Chen, Yuxuan and Shen, Pengbo and Yu, Hao and Zhang, Hanchen and Zhang, Xiaohan and Dong, Yuxiao and others},
  booktitle={Proceedings of the 30th ACM SIGKDD Conference on Knowledge Discovery and Data Mining},
  pages={5295--5306},
  year={2024}
}

@inproceedings{
    feng2026retool,
    title={ReTool: Reinforcement Learning for Strategic Tool Use in {LLM}s},
    author={Jiazhan Feng and Shijue Huang and Xingwei Qu and Ge Zhang and Yujia Qin and Baoquan Zhong and Chengquan Jiang and Jinxin Chi and Wanjun Zhong},
    booktitle={The Fourteenth International Conference on Learning Representations},
    year={2026},
    url={https://openreview.net/forum?id=tRk1nofSmz}
}

@inproceedings{
    xi2026agentgymrl,
    title={AgentGym-{RL}: An Open-Source Framework to Train {LLM} Agents for Long-Horizon Decision Making via Multi-Turn {RL}},
    author={Zhiheng Xi and Jixuan Huang and Chenyang Liao and Baodai Huang and Jiaqi Liu and Honglin Guo and yajie yang and Rui Zheng and Junjie Ye and Jiazheng Zhang and Wenxiang Chen and Wei He and Yiwen Ding and Guanyu Li and Zehui Chen and Zhengyin Du and Xuesong Yao and Yufei Xu and Jiecao Chen and Tao Gui and Zuxuan Wu and Qi Zhang and Xuanjing Huang and Yu-Gang Jiang},
    booktitle={The Fourteenth International Conference on Learning Representations},
    year={2026},
    url={https://openreview.net/forum?id=ZgCCDwcGwn}
}

@article{peng2026hiper,
  title={HiPER: Hierarchical Reinforcement Learning with Explicit Credit Assignment for Large Language Model Agents},
  author={Peng, Jiangweizhi and Liu, Yuanxin and Zhou, Ruida and Fleming, Charles and Wang, Zhaoran and Garcia, Alfredo and Hong, Mingyi},
  journal={arXiv preprint arXiv:2602.16165},
  year={2026}
}

@inproceedings{
    tang2026alphaagentevo,
    title={AlphaAgentEvo: Evolution-Oriented Alpha Mining via Self-Evolving Agentic Reinforcement Learning},
    author={Ziyi Tang and Xuexiong Yin and Weixing Chen and Zechuan Chen and Yongsen Zheng and Wenxuan Ye and Keze Wang and Liang Lin},
    booktitle={The Fourteenth International Conference on Learning Representations},
    year={2026},
    url={https://openreview.net/forum?id=lNmZrawUMu}
}

@article{skarlinski2024language,
  title={Language agents achieve superhuman synthesis of scientific knowledge},
  author={Skarlinski, Michael D and Cox, Sam and Laurent, Jon M and Braza, James D and Hinks, Michaela and Hammerling, Michael J and Ponnapati, Manvitha and Rodriques, Samuel G and White, Andrew D},
  journal={arXiv preprint arXiv:2409.13740},
  year={2024}
}

@article{schmidgall2024agentclinic,
  title={Agentclinic: a multimodal agent benchmark to evaluate ai in simulated clinical environments},
  author={Schmidgall, Samuel and Ziaei, Rojin and Harris, Carl and Reis, Eduardo and Jopling, Jeffrey and Moor, Michael},
  journal={arXiv preprint arXiv:2405.07960},
  year={2024}
}

@article{xie2024travelplanner,
  title={Travelplanner: A benchmark for real-world planning with language agents},
  author={Xie, Jian and Zhang, Kai and Chen, Jiangjie and Zhu, Tinghui and Lou, Renze and Tian, Yuandong and Xiao, Yanghua and Su, Yu},
  journal={arXiv preprint arXiv:2402.01622},
  year={2024}
}

@article{ning2025deeptravel,
  title={Deeptravel: An end-to-end agentic reinforcement learning framework for autonomous travel planning agents},
  author={Ning, Yansong and Liu, Rui and Wang, Jun and Chen, Kai and Li, Wei and Fang, Jun and Zheng, Kan and Tan, Naiqiang and Liu, Hao},
  journal={arXiv preprint arXiv:2509.21842},
  year={2025}
}

@inproceedings{
    jin2025searchr,
    title={Search-R1: Training {LLM}s to Reason and Leverage Search Engines with Reinforcement Learning},
    author={Bowen Jin and Hansi Zeng and Zhenrui Yue and Jinsung Yoon and Sercan O Arik and Dong Wang and Hamed Zamani and Jiawei Han},
    booktitle={Second Conference on Language Modeling},
    year={2025},
    url={https://openreview.net/forum?id=Rwhi91ideu}
}

@inproceedings{
    zheng2023judging,
    title={Judging {LLM}-as-a-Judge with {MT}-Bench and Chatbot Arena},
    author={Lianmin Zheng and Wei-Lin Chiang and Ying Sheng and Siyuan Zhuang and Zhanghao Wu and Yonghao Zhuang and Zi Lin and Zhuohan Li and Dacheng Li and Eric Xing and Hao Zhang and Joseph E. Gonzalez and Ion Stoica},
    booktitle={Thirty-seventh Conference on Neural Information Processing Systems Datasets and Benchmarks Track},
    year={2023},
    url={https://openreview.net/forum?id=uccHPGDlao}
}

@inproceedings{
    liu2023geval,
    title={G-Eval: {NLG} Evaluation using Gpt-4 with Better Human Alignment},
    author={Yang Liu and Dan Iter and Yichong Xu and Shuohang Wang and Ruochen Xu and Chenguang Zhu},
    booktitle={The 2023 Conference on Empirical Methods in Natural Language Processing},
    year={2023},
    url={https://openreview.net/forum?id=puMfaHb1hY}
}

@inproceedings{
    chan2024chateval,
    title={ChatEval: Towards Better {LLM}-based Evaluators through Multi-Agent Debate},
    author={Chi-Min Chan and Weize Chen and Yusheng Su and Jianxuan Yu and Wei Xue and Shanghang Zhang and Jie Fu and Zhiyuan Liu},
    booktitle={The Twelfth International Conference on Learning Representations},
    year={2024},
    url={https://openreview.net/forum?id=FQepisCUWu}
}

@inproceedings{
    ye2025justice,
    title={Justice or Prejudice? Quantifying Biases in {LLM}-as-a-Judge},
    author={Jiayi Ye and Yanbo Wang and Yue Huang and Dongping Chen and Qihui Zhang and Nuno Moniz and Tian Gao and Werner Geyer and Chao Huang and Pin-Yu Chen and Nitesh V Chawla and Xiangliang Zhang},
    booktitle={The Thirteenth International Conference on Learning Representations},
    year={2025},
    url={https://openreview.net/forum?id=3GTtZFiajM}
}

@article{RaR2025,
  publtype={informal},
  author={Anisha Gunjal and Anthony Wang and Elaine Lau and Vaskar Nath and Bing Liu and Sean Hendryx},
  title={Rubrics as Rewards: Reinforcement Learning Beyond Verifiable Domains},
  year={2025},
  month={July},
  cdate={1751328000000},
  journal={CoRR},
  volume={abs/2507.17746},
  url={https://doi.org/10.48550/arXiv.2507.17746}
}

@article{liu2025openrubrics,
  title={Openrubrics: Towards scalable synthetic rubric generation for reward modeling and llm alignment},
  author={Liu, Tianci and Xu, Ran and Yu, Tony and Hong, Ilgee and Yang, Carl and Zhao, Tuo and Wang, Haoyu},
  journal={arXiv preprint arXiv:2510.07743},
  year={2025}
}

@inproceedings{ou2026serl,
  title={Serl: Self-examining reinforcement learning on open-domain},
  author={Ou, Weixuan and Zheng, Yanzhao and Sun, Shuoshuo and Zhang, Wei and Dong, Baohua and Zhu, Hangcheng and Huang, Ruohui and Yu, Gang and Yan, Pengwei and Qiao, Yifan},
  booktitle={Proceedings of the AAAI Conference on Artificial Intelligence},
  volume={40},
  number={38},
  pages={32619--32627},
  year={2026}
}

@article{ye2025self,
  title={Self-Rewarding Rubric-Based Reinforcement Learning for Open-Ended Reasoning},
  author={Ye, Zhiling and Yue, Yun and Wang, Haowen and Han, Xudong and Jiang, Jiadi and Wei, Cheng and Fan, Lei and Liang, Jiaxin and Zhang, Shuowen and Li, Ji and others},
  journal={arXiv preprint arXiv:2509.25534},
  year={2025}
}

@article{xu2025direct,
  title={Direct Reasoning Optimization: Constrained RL with Token-Level Dense Reward and Rubric-Gated Constraints for Open-ended Tasks},
  author={Xu, Yifei and Chakraborty, Tusher and Sharma, Srinagesh and Nunes, Leonardo and Sharma, Swati and Demopulos, Kate Drakos and K{\i}c{\i}man, Emre and Lu, Songwu and Chandra, Ranveer},
  journal={arXiv preprint arXiv:2506.13351},
  year={2025}
}

@article{tan2025process,
  title={Process-Supervised Reinforcement Learning for Interactive Multimodal Tool-Use Agents},
  author={Tan, Weiting and Qu, Xinghua and Tu, Ming and Ge, Meng and Liu, Andy T and Koehn, Philipp and Lu, Lu},
  journal={arXiv preprint arXiv:2509.14480},
  year={2025}
}

@inproceedings{
    zhang2025personaagent,
    title={PersonaAgent: When Large Language Model Agents Meet Personalization at Test Time},
    author={Weizhi Zhang and Xinyang Zhang and Chenwei Zhang and Liangwei Yang and Jingbo Shang and Zhepei Wei and Henry Peng Zou and Zijie Huang and Zhengyang Wang and Yifan Gao and Xiaoman Pan and Lian Xiong and Jingguo Liu and Philip S. Yu and Xian Li},
    booktitle={First Workshop on Multi-Turn Interactions in Large Language Models},
    year={2025},
    url={https://openreview.net/forum?id=fgCOkyJG3f}
}

@inproceedings{
cai2025large,
title={Large Language Models Empowered Personalized Web Agents},
author={Hongru Cai and Yongqi Li and Wenjie Wang and Fengbin ZHU and Xiaoyu Shen and Wenjie Li and Tat-Seng Chua},
booktitle={THE WEB CONFERENCE 2025},
year={2025},
url={https://openreview.net/forum?id=kAzqfqsCC5}
}

@article{liang2026learning,
  title={Learning Personalized Agents from Human Feedback},
  author={Liang, Kaiqu and Kruk, Julia and Qian, Shengyi and Yang, Xianjun and Bi, Shengjie and Yao, Yuanshun and Nie, Shaoliang and Zhang, Mingyang and Liu, Lijuan and Fisac, Jaime Fern{\'a}ndez and others},
  journal={arXiv preprint arXiv:2602.16173},
  year={2026}
}

@inproceedings{ETAPP,
  title={Evaluating personalized tool-augmented llms from the perspectives of personalization and proactivity},
  author={Hao, Yupu and Cao, Pengfei and Jin, Zhuoran and Liao, Huanxuan and Chen, Yubo and Liu, Kang and Zhao, Jun},
  booktitle={Proceedings of the 63rd Annual Meeting of the Association for Computational Linguistics (Volume 1: Long Papers)},
  pages={21897--21935},
  year={2025}
}

@article{wang2025mem,
  title={O-mem: Omni memory system for personalized, long horizon, self-evolving agents},
  author={Wang, Piaohong and Tian, Motong and Li, Jiaxian and Liang, Yuan and Wang, Yuqing and Chen, Qianben and Wang, Tiannan and Lu, Zhicong and Ma, Jiawei and Jiang, Yuchen Eleanor and others},
  journal={arXiv preprint arXiv:2511.13593},
  year={2025}
}

@article{su2026beyond,
  title={Beyond Dialogue Time: Temporal Semantic Memory for Personalized LLM Agents},
  author={Su, Miao and Guo, Yucan and Hou, Zhongni and Bai, Long and Li, Zixuan and Zhang, Yufei and Yin, Guojun and Lin, Wei and Jin, Xiaolong and Guo, Jiafeng and others},
  journal={arXiv preprint arXiv:2601.07468},
  year={2026}
}

@article{chhikara2025mem0,
  title={Mem0: Building production-ready ai agents with scalable long-term memory},
  author={Chhikara, Prateek and Khant, Dev and Aryan, Saket and Singh, Taranjeet and Yadav, Deshraj},
  journal={arXiv preprint arXiv:2504.19413},
  year={2025}
}

@inproceedings{
zhou2026mem,
title={{MEM}1: Learning to Synergize Memory and Reasoning for Efficient Long-Horizon Agents},
author={Zijian Zhou and Ao Qu and Zhaoxuan Wu and Sunghwan Kim and Alok Prakash and Daniela Rus and Bryan Kian Hsiang Low and Paul Pu Liang},
booktitle={The Fourteenth International Conference on Learning Representations},
year={2026},
url={https://openreview.net/forum?id=XY8AaxDSLb}
}

@article{liu2026simplemem,
  title={SimpleMem: Efficient Lifelong Memory for LLM Agents},
  author={Liu, Jiaqi and Su, Yaofeng and Xia, Peng and Han, Siwei and Zheng, Zeyu and Xie, Cihang and Ding, Mingyu and Yao, Huaxiu},
  journal={arXiv preprint arXiv:2601.02553},
  year={2026}
}

@inproceedings{copd,
  title={Disentangling user interest and conformity for recommendation with causal embedding},
  author={Zheng, Yu and Gao, Chen and Li, Xiang and He, Xiangnan and Li, Yong and Jin, Depeng},
  booktitle={Proceedings of the web conference 2021},
  pages={2980--2991},
  year={2021}
}

@article{deepseekmath2024,
  title        = {DeepSeekMath: Pushing the Limits of Mathematical Reasoning in Open Language Models},
  author       = {Shao, Zhihong and Wang, Peiyi and Zhu, Qihao and Yang, Runxin and Xu, Junxiao and Wu, Ming and Li, Ning and Chen, Zhiyuan and Zhang, Kaiyue and Hu, Zhi and others},
  journal      = {arXiv preprint arXiv:2402.03300},
  year         = {2024},
  eprint       = {2402.03300},
  archivePrefix= {arXiv},
  primaryClass = {cs.CL}
}

@article{dapo2025,
  title={Dapo: An open-source llm reinforcement learning system at scale},
  author={Yu, Qiying and Zhang, Zheng and Zhu, Ruofei and Yuan, Yufeng and Zuo, Xiaochen and Yue, Yu and Dai, Weinan and Fan, Tiantian and Liu, Gaohong and Liu, Lingjun and others},
  journal={arXiv preprint arXiv:2503.14476},
  year={2025}
}

@article{gspo2025,
  title={Group sequence policy optimization},
  author={Zheng, Chujie and Liu, Shixuan and Li, Mingze and Chen, Xiong-Hui and Yu, Bowen and Gao, Chang and Dang, Kai and Liu, Yuqiong and Men, Rui and Yang, An and others},
  journal={arXiv preprint arXiv:2507.18071},
  year={2025}
}

@article{gdpo2026,
  title={Gdpo: Group reward-decoupled normalization policy optimization for multi-reward rl optimization},
  author={Liu, Shih-Yang and Dong, Xin and Lu, Ximing and Diao, Shizhe and Belcak, Peter and Liu, Mingjie and Chen, Min-Hung and Yin, Hongxu and Wang, Yu-Chiang Frank and Cheng, Kwang-Ting and others},
  journal={arXiv preprint arXiv:2601.05242},
  year={2026}
}

@inproceedings{
feng2026groupingroup,
title={Group-in-Group Policy Optimization for {LLM} Agent Training},
author={Lang Feng and Zhenghai Xue and Tingcong Liu and Bo An},
booktitle={The Thirty-ninth Annual Conference on Neural Information Processing Systems},
year={2026},
url={https://openreview.net/forum?id=QXEhBMNrCW}
}

@inproceedings{
    toolrl2025,
    title={Tool{RL}: Reward is All Tool Learning Needs},
    author={Cheng Qian and Emre Can Acikgoz and Qi He and Hongru WANG and Xiusi Chen and Dilek Hakkani-T{\"u}r and Gokhan Tur and Heng Ji},
    booktitle={The Thirty-ninth Annual Conference on Neural Information Processing Systems},
    year={2026},
    url={https://openreview.net/forum?id=eOLdGbXT6t}
}

@article{memrl2026,
  title={Memrl: Self-evolving agents via runtime reinforcement learning on episodic memory},
  author={Zhang, Shengtao and Wang, Jiaqian and Zhou, Ruiwen and Liao, Junwei and Feng, Yuchen and Li, Zhuo and Zheng, Yujie and Zhang, Weinan and Wen, Ying and Li, Zhiyu and others},
  journal={arXiv preprint arXiv:2601.03192},
  year={2026}
}

@article{skillrl2026,
  title={Skillrl: Evolving agents via recursive skill-augmented reinforcement learning},
  author={Xia, Peng and Chen, Jianwen and Wang, Hanyang and Liu, Jiaqi and Zeng, Kaide and Wang, Yu and Han, Siwei and Zhou, Yiyang and Zhao, Xujiang and Chen, Haifeng and others},
  journal={arXiv preprint arXiv:2602.08234},
  year={2026}
}

@article{personalalign2026,
  title={PersonalAlign: Hierarchical Implicit Intent Alignment for Personalized GUI Agent with Long-Term User-Centric Records},
  author={Lyu, Yibo and Chen, Gongwei and Shao, Rui and Guan, Weili and Nie, Liqiang},
  journal={arXiv preprint arXiv:2601.09636},
  year={2026}
}

@article{prefmemory2025,
  title={Preference-aware memory update for long-term llm agents},
  author={Sun, Haoran and Zhang, Zekun and Zeng, Shaoning},
  journal={arXiv preprint arXiv:2510.09720},
  year={2025}
}

\appendix
\section{Extended Related Work}
\label{app:related_work}

\subsection{RL Optimization for LLMs and Agentic Decision Making}

\paragraph{GDPO.}
GDPO~\citep{gdpo2026} is particularly relevant to our work from the perspective of multi-reward reinforcement learning. It shows that directly applying GRPO-style normalization to the \emph{summed} reward in multi-reward settings can lead to training signal collapse, where distinct reward combinations are mapped to nearly identical advantages, thereby reducing optimization resolution and harming convergence. To address this issue, GDPO decouples group-wise normalization across individual reward components before aggregation, leading to more fine-grained advantage estimates and substantially improved training stability.

This insight directly informs our design. In personalized Agentic RL, the reward is naturally multi-faceted: a trajectory should be evaluated not only by its generic task quality, but also by how well it aligns with a specific user’s preferences. Our PARPO follows the same high-level principle that reward components should not be naively pooled before normalization. However, we extend this idea to a more challenging personalized setting, where the reward components are not only heterogeneous in semantics, but also heterogeneous across users. This motivates our reward decoupling between generic quality and personalized preference, as well as our user-anchor calibration for stabilizing optimization under cross-user reward-scale mismatch.

\paragraph{GiGPO.}
GiGPO~\citep{feng2026groupingroup} extends group-relative policy optimization by introducing a group-in-group comparison structure for LLM agent training. Its key insight is that finer-grained grouping can improve optimization when trajectories are diverse and difficult to compare globally. This perspective is particularly relevant to personalized settings, where user heterogeneity makes pooled comparisons unreliable. Our work shares the motivation that structured comparison is important, but goes further by explicitly decoupling generic quality from personalized preference and introducing user-anchor calibration for cross-user reward heterogeneity.

\paragraph{SkillRL.}
SkillRL~\citep{skillrl2026} demonstrates that successful trajectories of LLM agents can be distilled into reusable skills and recursively integrated back into future policy learning. This work is central to our memory design because it shows that long-horizon agent improvement benefits from behavioral abstraction, rather than relying only on flat trajectory replay or prompt concatenation. Our PSGM is inspired by this principle of reusable skill evolution, but differs in that it organizes skills with explicit user-, scenario-, and trajectory-level structure. In this sense, SkillRL motivates why skills should be retained and reused, while our work studies how such skills should be personalized and retrieved under heterogeneous user preferences.

\paragraph{MemRL.}
MemRL~\citep{memrl2026} studies how memory retrieval and memory maintenance can themselves be optimized as part of agent behavior. Instead of treating memory as a passive storage backend, it casts memory access as an active component of sequential decision making. This is closely aligned with our perspective that personalization depends not only on what the agent currently observes, but also on what user-relevant history it can recover during rollout. However, MemRL mainly optimizes memory for general task utility, whereas our work focuses on structured personalized memory and on how such memory should support user-contingent policy optimization.

\subsection{Non-verifiable and Open-ended Reward Design}

\paragraph{OpenRubrics.}
OpenRubrics~\citep{liu2025openrubrics} addresses a major bottleneck in non-verifiable RL: how to obtain structured and scalable supervision when direct correctness signals are unavailable. By automatically synthesizing high-quality rubrics for open-ended tasks, it provides a more interpretable and decomposable supervision interface than scalar scores alone. This is relevant to our work because it reinforces the broader premise that open-ended agent tasks require richer reward design than verifiable domains do. However, OpenRubrics remains largely centered on generic evaluation criteria, whereas our work requires rewards that can differ across users even for the same action.

\subsection{Personalized Agents, Preference Modeling, and User-aware Memory}

\paragraph{PersonaAgent.}
PersonaAgent~\citep{zhang2025personaagent} is one of the most directly relevant prior works on personalized agents. It shows that test-time personalization can substantially improve agent behavior by combining personalized memory with user-specific action adaptation, demonstrating the practical importance of explicitly modeling user context rather than adopting a one-size-fits-all policy. Our work shares the same high-level motivation, but differs fundamentally in where personalization enters the system: PersonaAgent primarily personalizes at inference time, whereas we study how user specificity should be incorporated into the \emph{training-time} policy optimization objective. Thus, PersonaAgent motivates the need for personalization, while our framework turns personalization into a native RL optimization problem.

\paragraph{Mem0.}
Mem0~\citep{chhikara2025mem0} emphasizes scalable long-term memory for practical AI agents and shows that memory extraction, consolidation, and retrieval are crucial for maintaining useful cross-session context. Its core value for our work lies in the observation that user alignment depends on persistent historical information, rather than only on the current prompt. This insight directly supports our use of structured long-term memory for personalized agents. However, Mem0 is primarily a general-purpose memory infrastructure, while PSGM is explicitly designed as a preference-aligned skill memory that feeds into policy optimization, rather than serving only as a memory backend.

\paragraph{O-Mem.}
O-Mem~\citep{zhou2026mem} pushes personalized memory further by proposing an omni memory system for personalized, long-horizon, self-evolving agents. A particularly important contribution of O-Mem is its emphasis on active user profiling and hierarchical retrieval, which helps distinguish salient user characteristics from topical or contextual interaction records. This is closely related to our motivation for PSGM, since personalization requires not just storing more history, but organizing history in a way that supports targeted retrieval under long-horizon interaction. The key difference is that O-Mem mainly improves personalized consistency and memory quality at inference time, whereas our framework integrates structured personalized retrieval directly into rollout-time policy optimization.

\paragraph{Preference-Aware Memory Update.}
Preference-Aware Memory Update for Long-Term LLM Agents~\citep{prefmemory2025} focuses on a challenge that is highly relevant to our setting: user preferences are non-stationary, and long-term memories must be updated in a preference-sensitive way rather than simply accumulated. By explicitly modeling evolving user tendencies and refining memory representations over time, this work shows that static memory is insufficient for personalized agents. This insight informs our own design philosophy in two ways. First, it supports the view that user-aware memory must be adaptive rather than passive. Second, it motivates our preference-disentangled reward modeling, since changing or noisy user behavior should not be naively treated as a direct and stationary supervision signal.

\paragraph{CoPD.}
CoPD~\citep{copd} is particularly important to our work from the perspective of preference modeling. It shows that observed user behavior is often a mixture of genuine interest and conformity effects, implying that naive behavior-based learning can recover distorted preference signals. This observation directly inspires our two-stage preference-disentangled reward model: we similarly view user trajectories as noisy behavioral evidence in which true preference may be entangled with external, contextual, or socially induced factors. In our setting, this is critical because personalized policy optimization depends on the quality of user-conditioned rewards; if those rewards are derived from conflated behavior, the policy will optimize the wrong target.

\subsection{Summary of Positioning}

Taken together, these prior works each address an important but partial aspect of the problem. GDPO and GiGPO improve the stability and effectiveness of RL optimization for LLMs, especially in settings with heterogeneous or multi-component rewards; OpenRubrics extends reward design beyond strictly verifiable tasks; SkillRL and MemRL demonstrate the value of reusable skills and active memory in long-horizon agents; PersonaAgent, Mem0, O-Mem, and Preference-Aware Memory Update highlight the importance of personalization, long-term memory, and adaptive user modeling; and CoPD reveals that observed user behavior must be disentangled before it can serve as a reliable preference signal.

\section{Method Implementation Details}
\label{sec:B}
\subsection{Implementation-Aligned Formulation of the Personalized Preference Reward Model}
\label{app:reward_detail}

This appendix provides the implementation-aligned training objectives of the personalized preference reward model in Section~\ref{sec:reward}. The model is trained in two stages: multi-view profile representation learning and collaborative preference disentanglement.

\paragraph{Stage 1: Multi-view profile representation learning.}
For each user $u$, we decompose the profile into $K$ semantic views $\{x_u^{(k)}\}_{k=1}^{K}$ and encode them with a frozen sentence encoder:
\begin{equation}
\mathbf{h}_u^{(k)} = E(x_u^{(k)}).
\end{equation}
An attention-based fusion module produces the profile embedding:
\begin{equation}
\alpha_u^{(k)}=
\frac{
\exp\!\left(\mathbf{w}^{\top}\tanh(\mathbf{W}_{\mathrm{attn}}\mathbf{h}_u^{(k)}+\mathbf{b}_{\mathrm{attn}})\right)
}{
\sum_{k'=1}^{K}
\exp\!\left(\mathbf{w}^{\top}\tanh(\mathbf{W}_{\mathrm{attn}}\mathbf{h}_u^{(k')}+\mathbf{b}_{\mathrm{attn}})\right)
},
\end{equation}
\begin{equation}
\mathbf{u}_{\mathrm{profile}}
=
\mathrm{LayerNorm}\!\left(
\mathbf{W}_{\mathrm{out}}
\sum_{k=1}^{K}\alpha_u^{(k)}\mathbf{h}_u^{(k)}
\right).
\end{equation}

To make user embeddings discriminative, we optimize a user-level InfoNCE loss:
\begin{equation}
\mathcal{L}_{\mathrm{user\text{-}InfoNCE}}
=
-\sum_{u}
\log
\frac{
\exp\!\left(
\operatorname{sim}(\mathbf{u}_{\mathrm{profile}},\mathbf{u}^{+}_{\mathrm{profile}})/\tau_c
\right)
}{
\sum_{v}
\exp\!\left(
\operatorname{sim}(\mathbf{u}_{\mathrm{profile}},\mathbf{u}^{(v)}_{\mathrm{profile}})/\tau_c
\right)
}.
\end{equation}
To preserve view information, we reconstruct each encoded view from the fused profile representation:
\begin{equation}
\hat{\mathbf{h}}_u^{(k)}
=
\mathbf{W}_{\mathrm{rec}}^{(k)}\mathbf{u}_{\mathrm{profile}}
+
\mathbf{b}_{\mathrm{rec}}^{(k)},
\end{equation}
\begin{equation}
\mathcal{L}_{\mathrm{recon}}
=
\sum_u \sum_{k=1}^{K}
\left\|
\hat{\mathbf{h}}_u^{(k)}-\mathbf{h}_u^{(k)}
\right\|_2^2.
\end{equation}
The stage-1 objective is
\begin{equation}
\mathcal{L}_{\mathrm{stage1}}
=
\mathcal{L}_{\mathrm{user\text{-}InfoNCE}}
+
\lambda_{\mathrm{recon}}\mathcal{L}_{\mathrm{recon}}.
\end{equation}

\paragraph{Stage 2: LightGCN-based collaborative preference disentanglement.}
We build a user--item interaction graph and apply LightGCN:
\begin{equation}
\mathbf{E}^{(\ell+1)}=\hat{\mathbf{A}}\mathbf{E}^{(\ell)},
\qquad
\mathbf{E}_{\mathrm{final}}
=
\frac{1}{L+1}\sum_{\ell=0}^{L}\mathbf{E}^{(\ell)}.
\end{equation}
After splitting $\mathbf{E}_{\mathrm{final}}$, we obtain user and item collaborative embeddings $\mathbf{u}_{\mathrm{cf}}$ and $\mathbf{i}_{\mathrm{cf}}$.

The collaborative user embedding is disentangled into two branches:
\begin{equation}
\mathbf{u}_{\mathrm{int}}=\mathrm{InterestEncoder}(\mathbf{u}_{\mathrm{cf}}),\qquad
\mathbf{u}_{\mathrm{conf}}=\mathrm{ConformityEncoder}(\mathbf{u}_{\mathrm{cf}}).
\end{equation}
Let $\tilde p_i\in[0,1]$ denote the normalized popularity of item $i$. The implementation uses two popularity-weighted contrastive objectives.

For the interest branch, let
\begin{equation}
\omega_i^{\mathrm{int}}=\exp(1-\tilde p_i).
\end{equation}
Then
\begin{equation}
\mathcal{L}_{\mathrm{int}}
=
\frac{1}{B}\sum_{(u,i^{+})}
\left[
-\log\!\left(\omega_{i^{+}}^{\mathrm{int}}+\epsilon\right)
-
\frac{
\mathbf{u}_{\mathrm{int}}^{\top}\mathbf{i}_{\mathrm{cf}}^{+}
}{\tau}
+
\log \sum_{j}
\exp\!\left(
\frac{
\mathbf{u}_{\mathrm{int}}^{\top}\mathbf{i}_{\mathrm{cf}}^{(j)}
}{\tau}
\right)
\right].
\end{equation}

For the conformity branch, let
\begin{equation}
\omega_i^{\mathrm{conf}}=\exp(\tilde p_i).
\end{equation}
Then
\begin{equation}
\mathcal{L}_{\mathrm{conf}}
=
\frac{1}{B}\sum_{(u,i^{+})}
\left[
-\log\!\left(\omega_{i^{+}}^{\mathrm{conf}}+\epsilon\right)
-
\frac{
\mathbf{u}_{\mathrm{conf}}^{\top}\mathbf{i}_{\mathrm{cf}}^{+}
}{\tau}
+
\log \sum_{j}
\exp\!\left(
\frac{
\mathbf{u}_{\mathrm{conf}}^{\top}\mathbf{i}_{\mathrm{cf}}^{(j)}
}{\tau}
\right)
\right].
\end{equation}

\paragraph{Branch fusion and recommendation loss.}
The two branch embeddings are normalized and fused by branch attention:
\begin{equation}
[\alpha_{\mathrm{int}},\alpha_{\mathrm{conf}}]
=
\mathrm{softmax}\!\left(
\frac{
\mathbf{W}_2\,\mathrm{ReLU}\!\bigl(\mathbf{W}_1[\hat{\mathbf{u}}_{\mathrm{int}};\hat{\mathbf{u}}_{\mathrm{conf}}]\bigr)
}{
T
}
\right),
\end{equation}
\begin{equation}
\mathbf{u}_{\mathrm{fused}}
=
s\left(
\alpha_{\mathrm{int}}\hat{\mathbf{u}}_{\mathrm{int}}
+
\alpha_{\mathrm{conf}}\hat{\mathbf{u}}_{\mathrm{conf}}
\right).
\end{equation}
Using positive and negative items $(i^+,i^-)$, the BPR-style recommendation loss is
\begin{equation}
\mathcal{L}_{\mathrm{rec}}
=
\frac{1}{B}\sum_{(u,i^+,i^-)}
\operatorname{softplus}
\left(
\mathbf{u}_{\mathrm{fused}}^{\top}\mathbf{i}_{\mathrm{cf}}^{-}
-
\mathbf{u}_{\mathrm{fused}}^{\top}\mathbf{i}_{\mathrm{cf}}^{+}
\right).
\end{equation}

\paragraph{Orthogonality, user contrast, and regularization.}
To encourage branch specialization, we use an orthogonality penalty:
\begin{equation}
\mathcal{L}_{\mathrm{orth}}
=
\frac{1}{B}
\sum_{u}
\left(
\hat{\mathbf{u}}_{\mathrm{int}}^{\top}\hat{\mathbf{u}}_{\mathrm{conf}}
\right)^2.
\end{equation}

For user contrast, let $\{\mathbf{g}_m\}_{m=1}^{M}$ be the fused embeddings of the unique users in the current batch, normalized as $\hat{\mathbf{g}}_m$. We form the similarity matrix
\begin{equation}
S_{mn}
=
\frac{
\hat{\mathbf{g}}_m^{\top}\hat{\mathbf{g}}_n
}{\tau},
\end{equation}
and optimize
\begin{equation}
\mathcal{L}_{\mathrm{user}}
=
\mathrm{CE}(S,\mathbf{y}),
\end{equation}
where $\mathbf{y}=[1,2,\dots,M]$ denotes the identity target.

The $\ell_2$ regularizer is
\begin{equation}
\mathcal{L}_{\mathrm{reg}}
=
\frac{1}{2B}
\left(
\|\mathbf{u}_{\mathrm{cf}}\|_2^2
+
\|\mathbf{i}_{\mathrm{cf}}^{+}\|_2^2
+
\|\mathbf{i}_{\mathrm{cf}}^{-}\|_2^2
\right).
\end{equation}

\paragraph{Action alignment loss.}
When item text embeddings are available, the implementation additionally trains the action encoder to stay aligned with the collaborative filtering space. Let
\begin{equation}
\mathbf{q}^{+}=\mathrm{ActionEncoder}(E(i^{+})),\qquad
\mathbf{q}^{-}=\mathrm{ActionEncoder}(E(i^{-})).
\end{equation}
We first align the projected positive item text to the positive collaborative item embedding:
\begin{equation}
\mathcal{L}_{\mathrm{align\text{-}cos}}
=
\frac{1}{B}\sum_{(u,i^+)}
\left(
1-
\cos(\mathbf{q}^{+},\operatorname{sg}(\mathbf{i}_{\mathrm{cf}}^{+}))
\right),
\end{equation}
where $\operatorname{sg}(\cdot)$ denotes stop-gradient. We further add an action-space BPR term:
\begin{equation}
\mathcal{L}_{\mathrm{align\text{-}bpr}}
=
\frac{1}{B}\sum_{(u,i^+,i^-)}
\operatorname{softplus}
\left(
\mathbf{u}_{\mathrm{fused}}^{\top}\mathbf{q}^{-}
-
\mathbf{u}_{\mathrm{fused}}^{\top}\mathbf{q}^{+}
\right).
\end{equation}
Thus,
\begin{equation}
\mathcal{L}_{\mathrm{align}}
=
\mathcal{L}_{\mathrm{align\text{-}cos}}
+
\mathcal{L}_{\mathrm{align\text{-}bpr}}.
\end{equation}

\paragraph{Total stage-2 objective.}
The full stage-2 loss is
\begin{equation}
\mathcal{L}_{\mathrm{stage2}}
=
\mathcal{L}_{\mathrm{rec}}
+
\lambda_{\mathrm{int}}\mathcal{L}_{\mathrm{int}}
+
\lambda_{\mathrm{conf}}\mathcal{L}_{\mathrm{conf}}
+
\lambda_{\mathrm{orth}}\mathcal{L}_{\mathrm{orth}}
+
\lambda_{\mathrm{user}}\mathcal{L}_{\mathrm{user}}
+
\lambda_{\mathrm{reg}}\mathcal{L}_{\mathrm{reg}}
+
\lambda_{\mathrm{align}}\mathcal{L}_{\mathrm{align}}.
\end{equation}
In the current implementation, the default coefficients are
$\lambda_{\mathrm{int}}=0.2$,
$\lambda_{\mathrm{conf}}=0.2$,
$\lambda_{\mathrm{orth}}=0.1$,
$\lambda_{\mathrm{user}}=3.0$,
$\lambda_{\mathrm{reg}}=10^{-4}$,
and
$\lambda_{\mathrm{align}}=0.5$.

\paragraph{Implementation-time action alignment for reward inference.}
At inference time, the action text embedding is not used directly. Instead, the model constructs an aligned collaborative representation by nearest-neighbor retrieval over known item text embeddings. Given an action text embedding $\mathbf{e}_a$, let $\mathcal{N}_K(a)$ be its top-$K$ nearest items in text space, with weights
\begin{equation}
\pi_j
=
\mathrm{softmax}\!\left(
\frac{\operatorname{sim}(\mathbf{e}_a,\mathbf{e}_j)}{0.1}
\right).
\end{equation}
The aligned collaborative action embedding is
\begin{equation}
\mathbf{a}_{\mathrm{cf}}
=
\sum_{j\in\mathcal{N}_K(a)}
\pi_j \mathbf{i}_{\mathrm{cf}}^{(j)}.
\end{equation}
The final action embedding is
\begin{equation}
\mathbf{a}_{\mathrm{final}}
=
0.5\,\widehat{\mathbf{a}}_{\mathrm{cf}}
+
0.5\,\widehat{\mathbf{a}}_{\mathrm{proj}},
\end{equation}
where $\widehat{\mathbf{a}}_{\mathrm{cf}}$ and $\widehat{\mathbf{a}}_{\mathrm{proj}}$ are normalized embeddings. The deployed branch and fused scores are then
\begin{equation}
r_{\mathrm{int}}(u,a)=\hat{\mathbf{u}}_{\mathrm{int}}^{\top}\hat{\mathbf{a}}_{\mathrm{final}},\qquad
r_{\mathrm{conf}}(u,a)=\hat{\mathbf{u}}_{\mathrm{conf}}^{\top}\hat{\mathbf{a}}_{\mathrm{final}},
\end{equation}
\begin{equation}
r_{\mathrm{fused}}(u,a)=\hat{\mathbf{u}}_{\mathrm{fused}}^{\top}\hat{\mathbf{a}}_{\mathrm{final}}.
\end{equation}

\subsection{Implementation-Aligned Retrieval and Scoring in PSGM}
\label{app:psgm_detail}

This appendix formalizes the graph-based retrieval mechanism in Section~\ref{sec:psgm} according to the current implementation.

\paragraph{Semantic initialization and 2-hop expansion.}
Given query embedding $\mathbf{q}$, PSGM first retrieves an initial candidate set:
\begin{equation}
\mathcal{S}_{\mathrm{init}}(q)
=
\operatorname{TopM}_{s\in\mathcal{S}}
\cos(\mathbf{q},\mathbf{s}).
\end{equation}
For each candidate skill, the system performs a 2-hop expansion:
\begin{equation}
\text{skill} \rightarrow \text{owner user} \rightarrow \text{sibling skills},
\end{equation}
where the first hop traverses an incoming \textsc{OWNS} edge and the second hop traverses outgoing \textsc{OWNS} edges from the owner user node.

\paragraph{Scoring components.}
For a user node $u$ and candidate skill node $s$, the implementation uses
\begin{equation}
f_{\mathrm{sem}}(q,s)=\cos(\mathbf{q},\mathbf{s}),
\end{equation}
\begin{equation}
f_{\mathrm{user}}(u,s)=\cos(\mathbf{u},\mathbf{s}),
\end{equation}
\begin{equation}
f_{\mathrm{comm}}(u,s)=
\begin{cases}
1.0, & \text{if } u \text{ and } s \text{ belong to the same selected community},\\
0.3, & \text{if both are assigned but belong to different communities},\\
0, & \text{otherwise},
\end{cases}
\end{equation}
\begin{equation}
f_{\mathrm{comp}}(s)
=
1+\kappa \sum_{e\in\mathcal{E}_{\mathrm{comp}}(s)} w_e,
\end{equation}
\begin{equation}
f_{\mathrm{conf}}(s)
=
\min\!\left(
\sum_{e\in\mathcal{E}_{\mathrm{conf}}(s)} w_e,\;1.0
\right),
\end{equation}
where $\kappa$ is the complement boost factor from the graph configuration.

\paragraph{Final graph-aware score.}
The final score is
\begin{equation}
\operatorname{score}(q,s,u)
=
f_{\mathrm{sem}}(q,s)
\cdot
\bigl(\alpha+\beta f_{\mathrm{user}}(u,s)\bigr)
\cdot
\bigl(1+\gamma f_{\mathrm{comm}}(u,s)\bigr)
\cdot
f_{\mathrm{comp}}(s)
\cdot
\bigl(1-\delta f_{\mathrm{conf}}(s)\bigr).
\end{equation}
The hyperparameters $\alpha,\beta,\gamma,\delta$ are fixed in the graph configuration. In the current default implementation, they are set to $\alpha=0.3$, $\beta=0.3$, $\gamma=0.2$, and $\delta=0.7$.

\paragraph{Implementation note.}
The current code includes tool- and scenario-validation hooks in the retrieval pipeline, but these checks are implemented as permissive placeholders and therefore do not yet impose additional filtering beyond the graph-aware score above.

\subsection{Implementation-Aligned Dual-Track Optimization}
\label{app:parpo_detail}

This appendix formalizes the current dual-track policy optimization logic corresponding to Section~\ref{sec:parpo}.

\paragraph{Evaluation-time fused reward.}
Given an episode, the judge score is$R_{\mathrm{judge}}$
Let $r_{\mathrm{int}}$ and $r_{\mathrm{conf}}$ denote the aggregated branch-specific neural rewards for that episode. Using the reward statistics from the training distribution,
\begin{equation}
\tilde r_{\mathrm{int}}=\sigma\!\left(\frac{r_{\mathrm{int}}-\mu_{\mathrm{int}}}{\sigma_{\mathrm{int}}}\right),\qquad
\tilde r_{\mathrm{conf}}=\sigma\!\left(\frac{r_{\mathrm{conf}}-\mu_{\mathrm{conf}}}{\sigma_{\mathrm{conf}}}\right).
\end{equation}

\paragraph{Training-time dual-track decomposition.}
For PARPO-style policy optimization, we use $R_{\mathrm{base}}$ and $R_{\mathrm{pers}}$, which denote the generic reward and the personalized reward, respectively.

\paragraph{Base advantage.}
For prompt group $g$,
\begin{equation}
A_i^{\mathrm{base}}
=
\frac{
R_{\mathrm{base}}(\tau_i)-\bar R_{\mathrm{base}}^{(g)}
}{
\operatorname{Std}\!\bigl(\{R_{\mathrm{base}}(\tau_j)\}_{j\in g}\bigr)+\epsilon
}.
\end{equation}

\paragraph{User-anchor update.}
For each user $u$, the implementation stores a persistent anchor
\begin{equation}
\mathcal{A}_u = \{m_u, v_u, c_u\},
\end{equation}
where $m_u$ is the EMA mean, $v_u$ is the EMA variance, and $c_u$ counts the number of updates. Given the current batch personalized rewards for that user, the anchor is updated as
\begin{equation}
m_u \leftarrow
\begin{cases}
\bar R_{u}^{\mathrm{pers}}, & c_u=0,\\
\rho m_u + (1-\rho)\bar R_{u}^{\mathrm{pers}}, & c_u>0,
\end{cases}
\end{equation}
\begin{equation}
v_u \leftarrow
\begin{cases}
\max(\operatorname{Var}(\mathcal{R}_{u}^{\mathrm{pers}}),10^{-6}), & c_u=0,\\
\rho v_u + (1-\rho)\operatorname{Var}(\mathcal{R}_{u}^{\mathrm{pers}}), & c_u>0.
\end{cases}
\end{equation}

\paragraph{Personalized advantage.}
Let $\bar R_{\mathrm{pers}}^{(g)}$ be the within-group mean personalized reward. The personalized baseline is
\begin{equation}
b_{u,g}
=
\max\!\left(
\bar R_{\mathrm{pers}}^{(g)},
\,
m_u-\gamma_p\sqrt{v_u}
\right).
\end{equation}
The personalized advantage is
\begin{equation}
A_i^{\mathrm{pers}}
=
\frac{
R_{\mathrm{pers}}(\tau_i)-b_{u_i,g}
}{
\sqrt{v_{u_i}}+\epsilon
}.
\end{equation}

\paragraph{Fused advantage and PPO-style update.}
The total trajectory-level advantage is
\begin{equation}
A_i^{\mathrm{total}}
=
w_{\mathrm{base}}A_i^{\mathrm{base}}
+
w_{\mathrm{pers}}A_i^{\mathrm{pers}}.
\end{equation}
This advantage is broadcast to the token level and used in the standard PPO-style clipped policy loss:
\begin{equation}
\mathcal{L}_{\mathrm{policy}}
=
\frac{1}{B}\sum_i
\max\!\left(
-r_i(\theta)A_i^{\mathrm{total}},
-
\operatorname{clip}\!\bigl(r_i(\theta),1-\eta,1+\eta\bigr)A_i^{\mathrm{total}}
\right).
\end{equation}
When enabled in training, KL regularization is applied separately by the actor update loop rather than being absorbed into the advantage definition.
\section{Theoretical Analysis of PARPO}
\label{app:theory}

In this section, we provide a theoretical analysis of PARPO. Our goal is not to establish global convergence guarantees for the full PPO/GRPO training dynamics, but rather to address three questions: (1) why personalized optimization is preferable to user-agnostic optimization under heterogeneous user preferences; (2) why standard GRPO incurs structural bias in personalized settings; and (3) why PARPO can reduce such bias through reward decomposition, user-aware grouping, and user-specific anchor calibration.

Our analysis follows the same problem formulation as in the main text, but focuses on trajectory-level rewards and advantage estimation errors in order to highlight the key structure of personalized policy optimization.

Given a user profile $p_u \in \mathcal{P}$, a user query $q \in \mathcal{Q}$, and a trajectory $\tau$ generated by policy $\pi$, the total reward is defined as
\begin{equation}
R(\tau,p_u,q)
=
\alpha R_{\text{base}}(\tau,q)
+
(1-\alpha)R_{\text{pers}}(\tau,p_u,h_u),
\end{equation}
where $R_{\text{base}}$ denotes the general-quality reward, $R_{\text{pers}}$ denotes the personalized preference reward, and $h_u$ denotes the historical interaction records of user $u$.

For a fixed user $u$ and query $q$, we define the user-specific value function as
\begin{equation}
V_u(q)
=
\mathbb{E}_{\tau\sim\pi(\cdot\mid p_u,q)}
\left[
R(\tau,p_u,q)
\right].
\end{equation}
If user identities are ignored and rewards are pooled across users, we define the pooled value function as
\begin{equation}
V_{\mathrm{pool}}(q)
=
\mathbb{E}_{u,\tau}
\left[
R(\tau,p_u,q)\mid q
\right].
\end{equation}
Accordingly, the true personalized advantage of trajectory $\tau$ for user $u$ under query $q$ is
\begin{equation}
A_u^*(\tau\mid q)
=
R(\tau,p_u,q)-V_u(q).
\end{equation}

Since both GRPO and PARPO use normalized relative advantages, we further define the normalized user-specific advantage as
\begin{equation}
\bar A_u^*(\tau\mid q)
=
\frac{
R(\tau,p_u,q)-V_u(q)
}{
\sigma_u(q)+\epsilon
},
\end{equation}
where $\sigma_u(q)$ denotes the standard deviation of trajectory rewards for user $u$ under query $q$, and $\epsilon>0$ is a numerical stabilizer. Correspondingly, the pooled normalized advantage used by standard GRPO is written as
\begin{equation}
\bar A^{\mathrm{GRPO}}(\tau\mid q)
=
\frac{
R(\tau,p_u,q)-V_{\mathrm{pool}}(q)
}{
\sigma_{\mathrm{pool}}(q)+\epsilon
},
\end{equation}
where $\sigma_{\mathrm{pool}}(q)$ denotes the pooled reward standard deviation across users.

To characterize the heterogeneity induced by personalized preferences, we further define the mean personalized reward of user $u$ under query $q$ as
\begin{equation}
\mu_u(q)
=
\mathbb{E}_{\tau}
\left[
R_{\text{pers}}(\tau,p_u,h_u)\mid u,q
\right],
\end{equation}
the pooled mean personalized reward as
\begin{equation}
\mu_{\mathrm{pool}}(q)
=
\mathbb{E}_{u,\tau}
\left[
R_{\text{pers}}(\tau,p_u,h_u)\mid q
\right],
\end{equation}
and the mean personalized reward within the similar-user group $G(u)$ as
\begin{equation}
\mu_{G(u)}(q)
=
\mathbb{E}_{u'\in G(u),\tau}
\left[
R_{\text{pers}}(\tau,p_{u'},h_{u'})\mid q
\right].
\end{equation}

We also define the global heterogeneity measure
\begin{equation}
\mathcal H(q)
=
\mathbb{E}_u\left[
\left(
\mu_u(q)-\mu_{\mathrm{pool}}(q)
\right)^2
\right],
\end{equation}
and the local within-group heterogeneity measure
\begin{equation}
\mathcal H_G(q)
=
\mathbb{E}_u\left[
\left(
\mu_u(q)-\mu_{G(u)}(q)
\right)^2
\right].
\end{equation}

In PARPO, the personalized advantage does not directly use the true user-specific mean $\mu_u(q)$ as the baseline. Instead, it constructs an approximate user-specific baseline by combining similar-user group statistics and a user-specific historical anchor:
\begin{equation}
\tilde\mu_u(q)=\max(\mu_{G(u)}(q),\, b_u(q)-\epsilon_u),
\end{equation}
where $b_u(q)$ denotes the historical reward anchor of user $u$, and $\epsilon_u$ denotes an adaptive margin term. In the following analysis, we assume that there exists $\sigma_{\min}>0$ such that
\begin{equation}
\sigma_u(q)\ge \sigma_{\min},\qquad \sigma_{\mathrm{pool}}(q)\ge \sigma_{\min}.
\end{equation}

\subsection{Necessity of Personalized Optimization}

We first show that under heterogeneous user preferences, personalized decision-making is never worse than user-agnostic decision-making, and that the gain can be explicitly characterized by preference heterogeneity.

Consider a fixed query $q$ and two candidate trajectories $\tau_1,\tau_2$. For each user $u$, define
\begin{equation}
z_u = \mathbb{P}(\tau_1 \succ \tau_2 \mid u, q),
\end{equation}
namely, the probability that user $u$ prefers trajectory $\tau_1$ over $\tau_2$ under query $q$.

If the policy is user-agnostic, it must choose the same trajectory for all users. Its optimal value is therefore
\begin{equation}
V_{\mathrm{avg}}(q)
=
\max\big(\mathbb{E}_u[z_u],\, 1-\mathbb{E}_u[z_u]\big).
\end{equation}
If the policy is user-aware, it may choose different trajectories for different users, yielding
\begin{equation}
V_{\mathrm{pers}}(q)
=
\mathbb{E}_u\big[\max(z_u,1-z_u)\big].
\end{equation}

\begin{theorem}
\label{thm:pers_necessity}
For any fixed query $q$, we have
\begin{equation}
V_{\mathrm{pers}}(q)\ge V_{\mathrm{avg}}(q).
\end{equation}
Moreover, the gain can be written as
\begin{equation}
\Delta_{\mathrm{pers}}(q)
:=
V_{\mathrm{pers}}(q)-V_{\mathrm{avg}}(q)
=
\mathbb{E}_u\left[\left|z_u-\frac12\right|\right]
-
\left|\mathbb{E}_u[z_u]-\frac12\right|.
\end{equation}
Hence, when user preferences are heterogeneous, $\Delta_{\mathrm{pers}}(q)$ is typically positive and increases with the degree of preference disagreement.
\end{theorem}

\begin{proof}
Define
\begin{equation}
f(z)=\max(z,1-z)=\frac12+\left|z-\frac12\right|.
\end{equation}
Since the absolute value function is convex, $f(z)$ is also convex. By Jensen's inequality,
\begin{equation}
f(\mathbb{E}_u[z_u]) \le \mathbb{E}_u[f(z_u)].
\end{equation}
Substituting the definition of $f$, we obtain
\begin{equation}
\max\big(\mathbb{E}_u[z_u],\,1-\mathbb{E}_u[z_u]\big)
\le
\mathbb{E}_u[\max(z_u,1-z_u)],
\end{equation}
which gives
\begin{equation}
V_{\mathrm{avg}}(q)\le V_{\mathrm{pers}}(q).
\end{equation}
Furthermore, since $f(z)=\frac12+|z-\frac12|$, we directly obtain
\begin{equation}
V_{\mathrm{pers}}(q)-V_{\mathrm{avg}}(q)
=
\mathbb{E}_u\left[\left|z_u-\frac12\right|\right]
-
\left|\mathbb{E}_u[z_u]-\frac12\right|.
\end{equation}
This completes the proof.
\end{proof}

The above theorem shows not only that personalized decision-making is preferable to user-agnostic decision-making, but also that the gain is explicitly governed by the heterogeneity of user preferences. Therefore, the more diverse the user preferences are, the more necessary personalized optimization becomes.

\subsection{Structural Bias of Standard GRPO}

The previous result establishes the necessity of personalized optimization. We now explain why, even when personalized rewards are included, standard GRPO still incurs structural bias in personalized settings.

The pooled normalized advantage used by standard GRPO is
\begin{equation}
\bar A^{\mathrm{GRPO}}(\tau\mid q)
=
\frac{
R(\tau,p_u,q)-V_{\mathrm{pool}}(q)
}{
\sigma_{\mathrm{pool}}(q)+\epsilon
}.
\end{equation}
In contrast, the true normalized user-specific advantage should be
\begin{equation}
\bar A_u^*(\tau\mid q)
=
\frac{
R(\tau,p_u,q)-V_u(q)
}{
\sigma_u(q)+\epsilon
}.
\end{equation}
Therefore, the estimation error consists of both baseline mismatch and normalization-scale mismatch.

\begin{proposition}
\label{prop:grpo_structural_bias}
For any user $u$, query $q$, and trajectory $\tau$, the normalized advantage estimation error of standard GRPO satisfies
\begin{equation}
\left|
\bar A^{\mathrm{GRPO}}(\tau\mid q)-\bar A_u^*(\tau\mid q)
\right|
\le
\frac{
|V_u(q)-V_{\mathrm{pool}}(q)|
}{
\sigma_{\min}+\epsilon
}
+
\frac{
|R(\tau,p_u,q)-V_u(q)|\cdot |\sigma_u(q)-\sigma_{\mathrm{pool}}(q)|
}{
(\sigma_{\min}+\epsilon)^2
}.
\end{equation}
Moreover, by reward decomposition,
\begin{equation}
|V_u(q)-V_{\mathrm{pool}}(q)|
\le
\alpha |V_u^{\text{base}}(q)-V_{\mathrm{pool}}^{\text{base}}(q)|
+
(1-\alpha)|V_u^{\text{pers}}(q)-V_{\mathrm{pool}}^{\text{pers}}(q)|.
\end{equation}
If the general-quality term varies only mildly across users, then the dominant bias term of standard GRPO is controlled by the personalized heterogeneity $\mathcal H(q)$.
\end{proposition}

\begin{proof}
Taking the difference between the two normalized advantages yields
\begin{equation}
\bar A^{\mathrm{GRPO}}(\tau\mid q)-\bar A_u^*(\tau\mid q)
=
\frac{
R(\tau,p_u,q)-V_{\mathrm{pool}}(q)
}{
\sigma_{\mathrm{pool}}(q)+\epsilon
}
-
\frac{
R(\tau,p_u,q)-V_u(q)
}{
\sigma_u(q)+\epsilon
}.
\end{equation}
Let
\[
x=R(\tau,p_u,q),\quad a=V_{\mathrm{pool}}(q),\quad b=V_u(q),\quad s_1=\sigma_{\mathrm{pool}}(q),\quad s_2=\sigma_u(q).
\]
Then
\begin{equation}
\frac{x-a}{s_1+\epsilon}-\frac{x-b}{s_2+\epsilon}
=
\frac{(x-a)(s_2+\epsilon)-(x-b)(s_1+\epsilon)}{(s_1+\epsilon)(s_2+\epsilon)}.
\end{equation}
Rearranging the numerator gives
\begin{equation}
(x-a)(s_2+\epsilon)-(x-b)(s_1+\epsilon)
=
(x-b)(s_2-s_1)+(b-a)(s_2+\epsilon).
\end{equation}
Hence,
\begin{equation}
\left|
\bar A^{\mathrm{GRPO}}-\bar A_u^*
\right|
\le
\frac{|b-a|(s_2+\epsilon)}{(s_1+\epsilon)(s_2+\epsilon)}
+
\frac{|x-b||s_2-s_1|}{(s_1+\epsilon)(s_2+\epsilon)}.
\end{equation}
Using $s_1,s_2\ge \sigma_{\min}$, we obtain
\begin{equation}
\left|
\bar A^{\mathrm{GRPO}}-\bar A_u^*
\right|
\le
\frac{|V_u(q)-V_{\mathrm{pool}}(q)|}{\sigma_{\min}+\epsilon}
+
\frac{|R(\tau,p_u,q)-V_u(q)|\cdot |\sigma_u(q)-\sigma_{\mathrm{pool}}(q)|}{(\sigma_{\min}+\epsilon)^2}.
\end{equation}

By reward decomposition,
\begin{equation}
R=\alpha R_{\text{base}}+(1-\alpha)R_{\text{pers}},
\end{equation}
which implies
\begin{equation}
V_u(q)-V_{\mathrm{pool}}(q)
=
\alpha\big(
V_u^{\text{base}}(q)-V_{\mathrm{pool}}^{\text{base}}(q)
\big)
+
(1-\alpha)\big(
V_u^{\text{pers}}(q)-V_{\mathrm{pool}}^{\text{pers}}(q)
\big).
\end{equation}
Applying the triangle inequality yields
\begin{equation}
|V_u(q)-V_{\mathrm{pool}}(q)|
\le
\alpha |V_u^{\text{base}}(q)-V_{\mathrm{pool}}^{\text{base}}(q)|
+
(1-\alpha)|V_u^{\text{pers}}(q)-V_{\mathrm{pool}}^{\text{pers}}(q)|.
\end{equation}
This completes the proof.
\end{proof}

This proposition shows that the bias of standard GRPO has two sources: (1) baseline mismatch, since it uses $V_{\mathrm{pool}}(q)$ instead of the true $V_u(q)$; and (2) normalization-scale mismatch, since it uses $\sigma_{\mathrm{pool}}(q)$ instead of $\sigma_u(q)$. In personalized settings, the dominant bias typically comes from the personalized component, whose strength is characterized by
\begin{equation}
\mathcal H(q)
=
\mathbb{E}_u\left[
(\mu_u(q)-\mu_{\mathrm{pool}}(q))^2
\right].
\end{equation}
When $\mathcal H(q)$ is large, standard GRPO suffers from significant cross-user preference mixing bias.

\subsection{Bias Reduction Mechanism of PARPO}

We now analyze why PARPO can reduce the above bias. Since our implementation is centered on single-user anchor calibration, we first study the individual-level bias bound induced by user-specific historical anchors. We then present a more general extension in which anchor calibration is combined with local group statistics.

For completeness, we first write a generalized personalized advantage form that combines a batch-level statistic with a user-specific anchor:
\begin{equation}
A_i^{\text{pers}}
=
\frac{
R_{\text{pers}}(\tau_i,p_u,h_u)
-
\max\left(
\mathrm{mean}(\{R_{\text{pers}}(\tau_j,p_u,h_u)\}_{j=1}^{G}),
\; b_u-\epsilon_u
\right)
}{
\sigma_u+\epsilon
}.
\end{equation}
For theoretical analysis, we abstract it as
\begin{equation}
\bar A_{\text{pers}}^{\mathrm{P\text{-}GRPO}}(\tau\mid u,q)
=
\frac{
R_{\text{pers}}(\tau,p_u,h_u)-\tilde\mu_u(q)
}{
\sigma_u+\epsilon
},
\end{equation}
where the approximate user-specific baseline is
\begin{equation}
\tilde\mu_u(q)=\max(\mu_{G(u)}(q),\, b_u(q)-\epsilon_u).
\end{equation}
Correspondingly, the true normalized personalized advantage is
\begin{equation}
\bar A_{\text{pers}}^*(\tau\mid u,q)
=
\frac{
R_{\text{pers}}(\tau,p_u,h_u)-\mu_u(q)
}{
\sigma_u+\epsilon
}.
\end{equation}

We first consider an anchor-only abstraction that isolates the effect of user-specific historical calibration, where the personalized baseline is determined purely by the user's own anchor.

\begin{theorem}[Individual-level bias bound with user-specific anchor]
\label{thm:anchor_only_bias_bound}
For any user $u$, query $q$, and trajectory $\tau$, consider the anchor-calibrated personalized advantage
\begin{equation}
\bar A_{\mathrm{pers}}^{\mathrm{anchor}}(\tau\mid u,q)
=
\frac{
R_{\mathrm{pers}}(\tau,p_u,h_u)-(b_u(q)-\epsilon_u)
}{
\sigma_u(q)+\epsilon
}.
\end{equation}
If the user-specific historical anchor satisfies
\begin{equation}
|b_u(q)-\mu_u(q)|\le \delta_u,
\end{equation}
then
\begin{equation}
\left|
\bar A_{\mathrm{pers}}^{\mathrm{anchor}}(\tau\mid u,q)
-
\bar A_{\mathrm{pers}}^*(\tau\mid u,q)
\right|
\le
\frac{\delta_u+\epsilon_u}{\sigma_u(q)+\epsilon}.
\end{equation}
Moreover, in expectation over users,
\begin{equation}
\mathbb E_u
\left[
\left|
\bar A_{\mathrm{pers}}^{\mathrm{anchor}}(\tau\mid u,q)
-
\bar A_{\mathrm{pers}}^*(\tau\mid u,q)
\right|
\right]
\le
\frac{\bar\delta+\bar\epsilon}{\sigma_{\min}+\epsilon}.
\end{equation}
\end{theorem}

\begin{proof}
By definition,
\begin{equation}\small
\begin{aligned}
\bar A_{\mathrm{pers}}^{\mathrm{anchor}}(\tau\mid u,q)
-
\bar A_{\mathrm{pers}}^*(\tau\mid u,q)
&=
\frac{
R_{\mathrm{pers}}(\tau,p_u,h_u)-(b_u(q)-\epsilon_u)
}{
\sigma_u(q)+\epsilon
}
-
\frac{
R_{\mathrm{pers}}(\tau,p_u,h_u)-\mu_u(q)
}{
\sigma_u(q)+\epsilon
}
\\
&=
\frac{
\mu_u(q)-b_u(q)+\epsilon_u
}{
\sigma_u(q)+\epsilon
}.
\end{aligned}
\end{equation}
Hence,
\begin{equation}
\begin{aligned}
\left|
\bar A_{\mathrm{pers}}^{\mathrm{anchor}}(\tau\mid u,q)
-
\bar A_{\mathrm{pers}}^*(\tau\mid u,q)
\right|
&=
\frac{
|\mu_u(q)-b_u(q)+\epsilon_u|
}{
\sigma_u(q)+\epsilon
}
\\
&\le
\frac{
|b_u(q)-\mu_u(q)|+\epsilon_u
}{
\sigma_u(q)+\epsilon
}
\\
&\le
\frac{\delta_u+\epsilon_u}{\sigma_u(q)+\epsilon}.
\end{aligned}
\end{equation}

Taking expectation over users and using $\sigma_u(q)\ge \sigma_{\min}$ yields
\begin{equation}
\mathbb E_u
\left[
\left|
\bar A_{\mathrm{pers}}^{\mathrm{anchor}}
-
\bar A_{\mathrm{pers}}^*
\right|
\right]
\le
\frac{\bar\delta+\bar\epsilon}{\sigma_{\min}+\epsilon}.
\end{equation}
This completes the proof.
\end{proof}

We next consider a more general extension in which the user-specific anchor is further combined with local group statistics.

To analyze this more general group-augmented form of PARPO, we state the following conditions.

\paragraph{Condition 1 (Within-group heterogeneity contraction).}
Let $G(u)$ be the local user group constructed based on user-profile similarity. We call the grouping \emph{effective} if it reduces personalized preference heterogeneity, namely
\begin{equation}
\mathcal H_G(q)\le \mathcal H(q).
\end{equation}
This condition reflects the design goal of user-aware grouping: if users with similar preferences are grouped together, then the local group mean $\mu_{G(u)}(q)$ should be closer to the true user preference center $\mu_u(q)$ than the global pooled mean $\mu_{\mathrm{pool}}(q)$.

\paragraph{Condition 2 (Anchor consistency).}
We assume that the historical anchor $b_u(q)$ is a bounded-error estimate of the true user preference center $\mu_u(q)$, i.e.,
\begin{equation}
|b_u(q)-\mu_u(q)|\le \delta_u,
\end{equation}
and denote the average anchor error by
\begin{equation}
\bar\delta := \mathbb E_u[\delta_u] < \infty.
\end{equation}

\paragraph{Condition 3 (Bounded margin term).}
We further assume that the adaptive margin term is bounded in expectation:
\begin{equation}
\bar\epsilon := \mathbb E_u[\epsilon_u] < \infty.
\end{equation}

Under these conditions, the personalized normalized advantage estimation error of PARPO admits the following bound.

\begin{theorem}
\label{thm:parpo_bias_bound}
For any user $u$ and query $q$, the personalized normalized advantage estimation error of PARPO satisfies
\begin{equation}
\left|
\bar A_{\text{pers}}^{\mathrm{P\text{-}GRPO}}(\tau\mid u,q)
-
\bar A_{\text{pers}}^*(\tau\mid u,q)
\right|
\le
\frac{
|\mu_{G(u)}(q)-\mu_u(q)|+\delta_u+\epsilon_u
}{
\sigma_u+\epsilon
}.
\end{equation}
Moreover, in expectation over users,
\begin{equation}
\mathbb{E}_u
\left[
\left|
\bar A_{\text{pers}}^{\mathrm{P\text{-}GRPO}}(\tau\mid u,q)
-
\bar A_{\text{pers}}^*(\tau\mid u,q)
\right|
\right]
\le
\frac{
\sqrt{\mathcal H_G(q)}+\bar\delta+\bar\epsilon
}{
\sigma_{\min}+\epsilon
}.
\end{equation}
\end{theorem}

\begin{proof}
By definition,
\begin{equation}
\begin{aligned}\small
\bar A_{\text{pers}}^{\mathrm{P\text{-}GRPO}}(\tau\mid u,q)
-
\bar A_{\text{pers}}^*(\tau\mid u,q)
&=
\frac{
R_{\text{pers}}(\tau,p_u,h_u)-\tilde\mu_u(q)
}{
\sigma_u+\epsilon
}
-
\frac{
R_{\text{pers}}(\tau,p_u,h_u)-\mu_u(q)
}{
\sigma_u+\epsilon
}
\\
&=
\frac{
\mu_u(q)-\tilde\mu_u(q)
}{
\sigma_u+\epsilon
}.
\end{aligned}
\end{equation}
Hence,
\begin{equation}
\left|
\bar A_{\text{pers}}^{\mathrm{P\text{-}GRPO}}(\tau\mid u,q)
-
\bar A_{\text{pers}}^*(\tau\mid u,q)
\right|
=
\frac{
|\tilde\mu_u(q)-\mu_u(q)|
}{
\sigma_u+\epsilon
}.
\end{equation}

Since
\begin{equation}
\tilde\mu_u(q)=\max(\mu_{G(u)}(q),\, b_u(q)-\epsilon_u),
\end{equation}
the triangle inequality gives
\begin{equation}
|\tilde\mu_u(q)-\mu_u(q)|
\le
|\mu_{G(u)}(q)-\mu_u(q)|
+
|b_u(q)-\mu_u(q)|
+
\epsilon_u.
\end{equation}
By Condition 2,
\begin{equation}
|\tilde\mu_u(q)-\mu_u(q)|
\le
|\mu_{G(u)}(q)-\mu_u(q)|+\delta_u+\epsilon_u.
\end{equation}
Substituting this into the previous expression yields
\begin{equation}
\left|
\bar A_{\text{pers}}^{\mathrm{P\text{-}GRPO}}(\tau\mid u,q)
-
\bar A_{\text{pers}}^*(\tau\mid u,q)
\right|
\le
\frac{
|\mu_{G(u)}(q)-\mu_u(q)|+\delta_u+\epsilon_u
}{
\sigma_u+\epsilon
}.
\end{equation}

Taking expectation over users and using $\sigma_u\ge \sigma_{\min}$, we obtain
\begin{equation}
\mathbb{E}_u
\left[
\left|
\bar A_{\text{pers}}^{\mathrm{P\text{-}GRPO}}
-
\bar A_{\text{pers}}^*
\right|
\right]
\le
\frac{
\mathbb{E}_u[|\mu_{G(u)}(q)-\mu_u(q)|]+\bar\delta+\bar\epsilon
}{
\sigma_{\min}+\epsilon
}.
\end{equation}
By Cauchy--Schwarz,
\begin{equation}
\mathbb{E}_u[|\mu_{G(u)}(q)-\mu_u(q)|]
\le
\sqrt{
\mathbb{E}_u[(\mu_{G(u)}(q)-\mu_u(q))^2]
}
=
\sqrt{\mathcal H_G(q)}.
\end{equation}
Therefore,
\begin{equation}
\mathbb{E}_u
\left[
\left|
\bar A_{\text{pers}}^{\mathrm{P\text{-}GRPO}}
-
\bar A_{\text{pers}}^*
\right|
\right]
\le
\frac{
\sqrt{\mathcal H_G(q)}+\bar\delta+\bar\epsilon
}{
\sigma_{\min}+\epsilon
}.
\end{equation}
This completes the proof.
\end{proof}

\begin{corollary}
\label{cor:contraction_rate_comparison}
To further compare the upper bounds of PARPO and standard GRPO, define the local heterogeneity contraction ratio as
\begin{equation}
\rho(q)=\frac{\mathcal H_G(q)}{\mathcal H(q)},
\qquad 0<\rho(q)\le 1,
\end{equation}
and define the anchor--margin residual as
\begin{equation}
\eta(q)=\bar\delta+\bar\epsilon.
\end{equation}
Here, $\rho(q)$ quantifies how much user-aware grouping compresses cross-user preference heterogeneity, while $\eta(q)$ measures the additional error introduced by anchor estimation and conservative margins.

Theorem~\ref{thm:parpo_bias_bound} should be viewed as a generalized extension of the anchor-calibrated analysis above. In the main implementation studied in this paper, the core practical mechanism is the user-specific historical anchor itself, while the group-augmented form provides a broader perspective on how local preference statistics can further reduce bias when such structure is available.

By Theorem~\ref{thm:parpo_bias_bound}, we have
\begin{equation}
\mathbb E_u
\left[
\left|
\bar A_{\text{pers}}^{\mathrm{P\text{-}GRPO}}(\tau\mid u,q)
-
\bar A_{\text{pers}}^*(\tau\mid u,q)
\right|
\right]
\le
\frac{
\sqrt{\rho(q)\mathcal H(q)}+\eta(q)
}{
\sigma_{\min}+\epsilon
}.
\end{equation}
On the other hand, the dominant bias term of standard GRPO is controlled by the global heterogeneity:
\begin{equation}
\mathbb E_u
\left[
\left|
\bar A^{\mathrm{GRPO}}(\tau\mid q)
-
\bar A_u^*(\tau\mid q)
\right|
\right]
\lesssim
\frac{
\sqrt{\mathcal H(q)}
}{
\sigma_{\min}+\epsilon
}.
\end{equation}
If the grouping is strictly contractive, i.e., $\rho(q)<1$, and the residual satisfies
\begin{equation}
\eta(q)\le (1-\sqrt{\rho(q)})\sqrt{\mathcal H(q)},
\end{equation}
then
\begin{equation}
\sqrt{\rho(q)\mathcal H(q)}+\eta(q)\le \sqrt{\mathcal H(q)}.
\end{equation}
Consequently,
\begin{equation}
\mathbb E_u
\left[
\left|
\bar A_{\text{pers}}^{\mathrm{P\text{-}GRPO}}(\tau\mid u,q)
-
\bar A_{\text{pers}}^*(\tau\mid u,q)
\right|
\right]
\le
\mathbb E_u
\left[
\left|
\bar A^{\mathrm{GRPO}}(\tau\mid q)
-
\bar A_u^*(\tau\mid q)
\right|
\right]
\end{equation}
at the level of dominant heterogeneity terms.

This corollary shows that the advantage of PARPO over standard GRPO does not come merely from adding a personalized reward. Rather, it comes from a structural trade-off: the heterogeneity reduction achieved by user-aware grouping can outweigh the additional error introduced by anchor estimation and conservative margins. In this sense, the improvement of PARPO is structural rather than heuristic.

In particular, when the quality of user-profile clustering improves, $\rho(q)$ becomes smaller; when richer user histories are available, the anchor estimation error $\bar\delta$ typically also decreases. Therefore, the theoretical improvement margin of PARPO increases jointly with grouping quality and anchor estimation quality.
\end{corollary}

Theorem~\ref{thm:anchor_only_bias_bound}, Theorem~\ref{thm:parpo_bias_bound}, and Corollary~\ref{cor:contraction_rate_comparison} jointly explain the key design components of PARPO from the perspective of bias control:
\begin{itemize}
    \item \textbf{Reward decomposition}: It explicitly separates the general-quality reward from the personalized preference reward, so that user heterogeneity is handled mainly within $R_{\text{pers}}$ rather than being mixed into a single reward scale.
    \item \textbf{User-specific anchor calibration}: In the implementation studied in this paper, it uses $b_u(q)$ as a stable user-level historical reference, avoiding over-reliance on pooled cross-user statistics and reducing user-specific baseline mismatch.
    \item \textbf{Group-augmented extension}: When reliable local preference structure is available, the local group mean $\mu_{G(u)}(q)$ can further refine the personalized baseline, thereby shrinking global heterogeneity $\mathcal H(q)$ into local heterogeneity $\mathcal H_G(q)$.
\end{itemize}

These results clarify how PARPO controls personalized advantage estimation bias in both the anchor-centered setting studied in this paper and the more general group-augmented extension.

\subsection{Summary}

In summary, the theoretical analysis leads to the following three conclusions:
\begin{enumerate}
    \item Under heterogeneous user preferences, the optimal value of personalized decision-making is no smaller than that of user-agnostic decision-making, and the gain can be explicitly characterized by preference dispersion. This justifies the necessity of personalized optimization.
    \item In personalized settings, standard GRPO uses pooled baselines and pooled normalization scales for relative comparison, which introduces cross-user preference mixing bias governed by the global heterogeneity $\mathcal H(q)$.
    \item In the implementation studied in this paper, PARPO reduces this bias primarily through reward decomposition and user-specific anchor calibration, yielding an individual-level personalized advantage-estimation bound controlled by anchor error and conservative margin. More generally, when local preference grouping is available, PARPO can further shrink the dominant error term from global heterogeneity $\mathcal H(q)$ to local heterogeneity $\mathcal H_G(q)$.

\end{enumerate}

Therefore, PARPO is not merely GRPO with an additional personalized reward term. Instead, it moves toward user-conditioned policy optimization at both the objective level and the advantage-estimation level. In the setting studied in this paper, this improvement is driven primarily by user-specific anchor calibration, while group-based refinement provides a more general extension when local preference structure is available.

\section{Additional Experimental Details}

\subsection{Compute Resources}
\label{app:compute}

Unless otherwise specified, our experiments were conducted on NVIDIA H100 GPUs. For experiments using 8B models, we used 8 H100 GPUs. For experiments using 4B models, we used 4 H100 GPUs. Additional implementation and runtime configuration details, including model loading and execution parameters, are provided in the released codebase.

\subsection{Environment Details}
\label{app:env}

\paragraph{ETAPP.}
ETAPP is a public benchmark for personalized personal-assistant agents. It evaluates whether an agent can solve daily-life tasks while adapting to user-specific preferences and behaviors. The benchmark contains 16 user personas, each associated with a structured profile and environment state. User profiles include demographic information and preferences across multiple aspects such as entertainment, lifestyle, technology, exercise, shopping, and travel. In addition, each user is associated with tool-specific profile information that captures fine-grained behavioral preferences.

The environment covers a set of daily-life scenarios, including scheduling, email management, health monitoring, music, navigation, shopping, smart-home control, weather inquiry, and web browsing. These functionalities are instantiated through tool APIs together with sandbox databases, enabling the agent to interact with a realistic but controlled user environment.

\paragraph{ETAPP-Hard.}
ETAPP-hard is a more challenging split constructed on top of the same ETAPP environment. Compared with the original ETAPP split, ETAPP-Hard places stronger demands on multi-tool coordination, personalization, proactive behavior, and implicit reasoning. Hard instances are designed to require the agent to jointly reason over 3--5 tool categories, infer unstated constraints from user profiles or environment states, and resolve trade-offs through multi-step decision-making.

The final ETAPP-Hard dataset used in this paper contains 200 synthesized hard instructions. After expanding them over 16 user personas, we obtain 3,200 total instances, which are split into 2,880 training examples and 320 test examples.

\paragraph{SJAgent.}
SJAgent is a realistic merchant-assistant environment built from merchant data on 1688, a major Chinese B2B e-commerce platform. It is designed for merchant-side decision-making and recommendation tasks, including product-opportunity diagnosis, industry insight, competitive analysis, and selection recommendation. Given a merchant query together with merchant profiles, historical interactions, and candidate products, the agent performs multi-step analysis and outputs a structured report.

The environment is organized as a long-horizon analytical pipeline. A planner first selects relevant skills and decomposes the task into a directed acyclic graph of sub-steps, after which the system retrieves market evidence, performs intermediate analyses, and generates the final report. The underlying skill library contains 18 specialized analytic skills covering industry trend analysis, ranking analysis, competition diagnosis, pricing and sales forecast, merchant--category matching, review mining, audience profiling, risk screening, and trend analysis.

\subsection{ETAPP-Hard Construction Pipeline}
\label{app:etapp-hard}

ETAPP-Hard is constructed through a two-stage pipeline that synthesizes hard user instructions and converts them into training instances for agent learning.

\paragraph{Overview.}
The goal of ETAPP-Hard is to create more demanding personalized assistant tasks than those in the original ETAPP split. Compared with the original benchmark, ETAPP-Hard emphasizes multi-tool orchestration, deep personalization, implicit constraints, multi-step reasoning, and conflict resolution.

\paragraph{Source data.}
The synthesis pipeline uses the original ETAPP instruction set, 16 user profiles, user-specific environment databases, shared environment resources, and tool definitions. To ensure grounded generation, each synthesis call includes both a user profile summary and sampled rows from the relevant databases.

\paragraph{Stage I: hard instruction synthesis.}
In the first stage, we synthesize 200 hard instructions with a large language model. For each synthesis attempt, we randomly sample a user, a timestamp, a location, and a cross-category tool combination. We then retrieve several seed instructions with overlapping tool usage as inspirations, sample concrete tools from the selected categories, and construct a prompt containing the user profile, sampled environment states, current time, location, and target tools.

The language model is instructed to generate a natural but challenging user request together with personalization and proactivity keypoints. To control difficulty, we predefine combinations involving 3, 4, or 5 tool categories. The default synthesis setting uses a temperature of 0.85 and a random seed of 42.

\paragraph{Stage II: conversion to training records.}
In the second stage, each synthesized hard instruction is expanded over all 16 user personas to form training instances. For each instance, we assemble the corresponding tool schemas, construct a coarse user-status description from the timestamp, and fill the ETAPP system template with user profile information, preferences, current status, available tools, and the task description. The resulting dataset is shuffled with seed 42 and split into train and test subsets with a ratio of 9:1.

\paragraph{Dataset scale.}
The final ETAPP-Hard dataset contains 200 hard instructions and 3,200 total instances after expansion over 16 user personas. We use 2,880 instances for training and 320 for testing.

\paragraph{Prompt templates.}
Below we include the core prompts used in the ETAPP-hard synthesis pipeline.

\begin{promptbox}{Prompt A.1: ETAPP-Hard synthesis system prompt}
You are an expert at designing challenging evaluation tasks for AI personal assistants.

Your goal is to create HARD, COMPLEX instructions that test an AI assistant's ability to:
\begin{enumerate}
    \item Multi-tool orchestration: require using 3--5 different tool categories together.
    \item Deep personalisation: the assistant must deeply leverage the user's profile, preferences, and current data.
    \item Implicit constraints: the user does NOT explicitly state all constraints; the assistant must infer them from context, such as schedule conflicts, dietary restrictions, budget limits, and health conditions.
    \item Multi-step reasoning: information from one tool call is needed to decide what to do with another tool.
    \item Conflict resolution: the task involves trade-offs or requires the assistant to propose alternatives.
\end{enumerate}

Important rules:
\begin{itemize}
    \item The query should sound natural, like a real person talking to their AI assistant.
    \item The query should be 1--3 sentences, not a detailed specification.
    \item The complexity should come from the context rather than the query length.
    \item The query should not mention specific tool names.
    \item The query should be in English.
    \item Each instruction must include keypoints explaining the expected personalization and proactivity.
    \item Output valid JSON only, with no extra text before or after the JSON object.
\end{itemize}
\end{promptbox}

\begin{promptbox}{Prompt A.2: ETAPP-Hard synthesis user prompt template}
Generate ONE hard/complex instruction for the user "\{user\_name\}".

\medskip
\noindent\#\# User Profile

\smallskip
\{profile\_summary\}

\medskip
\noindent\#\# User's Current Database State (sample data)

\smallskip
\{db\_context\_str\}

\medskip
\noindent\#\# Context

\smallskip
\begin{itemize}
    \item Current time: \{timestamp\}
    \item Location: \{location\}
    \item Tool categories to involve: \{list\_of\_combo\_categories\}
\end{itemize}

\medskip
\noindent\#\# Seed Instructions for Inspiration (simpler versions --- your output should be HARDER)

\smallskip
\{seed\_queries\}

\medskip
\noindent\#\# Available Tools for This Task

\smallskip
\{tools\_json\}

\medskip
\noindent\#\# Output Format (strict JSON only, no markdown fences)

\smallskip
Return a JSON object with the following fields:
\begin{itemize}
    \item \texttt{"timestamp"}: \texttt{"\{timestamp\}"}
    \item \texttt{"query"}: the natural-language user query (1--3 sentences)
    \item \texttt{"keypoint for personal"}: a list of personalization keypoints
    \item \texttt{"keypoint for proactive"}: a list of proactive keypoints
    \item \texttt{"available\_tools\_name"}: \texttt{\{tools\_json\}}
    \item \texttt{"location"}: \texttt{"\{location\}"}
    \item \texttt{"difficulty"}: \texttt{"hard"}
    \item \texttt{"complexity\_reason"}: a brief explanation of why this instruction is complex
\end{itemize}

\medskip
\noindent Important:
\begin{itemize}
    \item The query must sound natural and conversational.
    \item Include 3--5 keypoints for personal and 3--6 keypoints for proactive.
    \item Each keypoint should reference specific tools by name.
    \item Ground the query in the user's actual database state shown above.
    \item Output only the JSON object and nothing else.
\end{itemize}
\end{promptbox}

\subsection{Evaluation Metrics and Official Prompts}
\label{app:eval}

\paragraph{ETAPP and ETAPP-Hard.}
ETAPP and ETAPP-Hard are evaluated with four metrics: Procedure, Personal, Proactive, and Judge. Procedure evaluates whether the agent follows a correct and coherent problem-solving process, including task completion, avoidance of unnecessary actions, accurate tool use, and clear summarization. Personal evaluates whether the agent appropriately reflects the user's profile and preferences. Proactive evaluates whether the agent provides meaningful assistance beyond the explicitly stated request, such as identifying latent needs, surfacing risks, or proposing useful alternatives.

Judge is the normalized aggregate score:
\[
\mathrm{Judge} = \frac{\mathrm{Procedure} + \mathrm{Personal} + \mathrm{Proactive}}{15},
\]
which ranges from 0 to 1. The three fine-grained dimensions are scored on a 0--5 scale by an LLM judge conditioned on the user query, user profile, the assistant trajectory, and instance-level keypoints. In the original ETAPP split, these keypoints are manually annotated; in ETAPP-Hard, they are synthesized together with the instructions.

\paragraph{Official ETAPP evaluation prompt.}
The official ETAPP evaluation prompt defines the three dimensions above, gives detailed scoring criteria, and requires the evaluator to return a structured JSON output containing both explanations and final scores. The prompt also specifies that successful tool usage must be supported by actual tool feedback rather than self-inferred answers, so invalid or hallucinated tool invocations are penalized.

\begin{promptbox}{Prompt A.3: Official ETAPP evaluation prompt template}
I need you to evaluate whether the solution provided by my artificial intelligence assistant completes user instructions and meets user preferences.

\medskip
\noindent Evaluation Metrics:
\begin{itemize}
    \item Procedure
    \item Personalization
    \item Proactivity
\end{itemize}

\medskip
\noindent Evaluation Guidelines:
\begin{enumerate}
    \item Procedure Analysis: assess the AI assistant's entire solution process, including tool usage, logic, and final output.
    \item Personalization Assessment: evaluate whether the assistant considered the user's specific preferences, profile details, and context.
    \item Proactivity Behavior Assessment: evaluate whether the assistant anticipated additional needs or proposed meaningful helpful actions.
\end{enumerate}

\medskip
\noindent Analysis Format:

\smallskip
Your analysis should follow this JSON structure:

\smallskip
\{output\_format\}

\medskip
\noindent Evaluation Input:
\begin{enumerate}
    \item User Query: \{query\}
    \item User Profile: \{profile\}
    \item Personal LLM Assistant Solution: \{output\}
\end{enumerate}
\end{promptbox}

\paragraph{SJAgent.}
SJAgent is evaluated with an LLM-as-a-judge protocol. The judge reads the full trajectory, including the merchant query, merchant profile, planner output, retrieved evidence, intermediate analyses, and final report, and assigns five scores, each ranging from 0 to 4: Data Authenticity, Business Logic, Merchant Profile Match, Task Completion, and Market Analysis Depth. The final scalar reward is defined as
\[
r = \frac{1}{20}\sum_{i=1}^{5} D_i \in [0,1].
\]

Data Authenticity measures whether factual claims are grounded in traceable evidence. Business Logic evaluates whether the recommendation follows a coherent evidence-to-judgment-to-recommendation chain. Merchant Profile Match measures personalization. Task Completion evaluates whether the report fully closes the analysis-to-recommendation loop. Market Analysis Depth evaluates the breadth and depth of market analysis. The official SJAgent judge prompt follows a structured scoring rubric and returns both per-dimension scores and detailed rationales. If parsing or evaluation fails, the reward is set to zero.

\subsection{Training Details}
\label{app:train}

\paragraph{SJAgent.}
For SJAgent, only the planner model is optimized during reinforcement learning. The analyzer, report generator, and execution engine remain fixed and continue to call an external frozen model endpoint. We experiment with Qwen-based planner models of different scales, including a 4B default setting and an 8B larger variant.

Our RL training uses a shared configuration across algorithmic variants to ensure fair comparison. The default setup uses a training batch size of 32, a validation batch size of 64, a maximum prompt length of 15,000, a maximum response length of 4,000, and 4 rollout samples per prompt. Training is conducted on 4 H100 GPUs on a single node, and the actor is optimized with AdamW using a learning rate of $1\times10^{-6}$.

To improve determinism and efficiency, SJAgent uses a skill-cache mechanism that intercepts skill execution and serves results from a pre-warmed cache instead of repeatedly issuing live backend requests. All RL-family baselines on SJAgent are implemented through the same training pipeline, with algorithm-specific differences controlled by configuration switches.

\paragraph{ETAPP-Hard.}
For ETAPP-Hard, the hard training set is generated by the two-stage construction pipeline described above and converted into RLHFDataset parquet files. The ETAPP environment uses a unified tool interface based on OpenAI-style function schemas, with a maximum interaction budget of 20 turns per episode. The ETAPP-Hard dataset contains 2,880 training instances and 320 test instances.

\paragraph{Implementation consistency.}
Across all experiments, we keep model scales, benchmark settings, tool interfaces, and evaluation protocols aligned whenever possible. Closed-source baselines are evaluated under the same benchmark settings as open-source methods, and open-source baselines share the same environment setup for each benchmark.

\section{Additional Reinforcement Learning Analysis on ETAPP-Hard}
\label{app:hard_rl_analysis}

In this appendix, we provide additional reinforcement learning results on the more challenging ETAPP-Hard benchmark.
Consistent with the main text, all reinforcement learning experiments in this section are conducted without using our personalized reward model or Skill Graph.
Therefore, the analysis here is not intended as a comparison between full systems, but rather to further examine the behavior of different RL optimization methods and reward optimization objectives under a more difficult setting.

\subsection{Comparison of RL Training Strategies on ETAPP-Hard}
\label{app:hard_rl_training}

We first compare GRPO, GSPO, GiGPO, and PARPO on ETAPP-Hard.
Since ETApp-Hard involves more complex tasks and more challenging interaction scenarios, it provides a stricter test of optimization stability, planning ability, and the ability to carry out effective tool use in multi-step environments.

As shown in Figure~\ref{fig:rl_training_dynamics_etapp_hard}, PARPO still exhibits the best overall training dynamics even under the more challenging ETAPP-Hard setting.
In particular, PARPO consistently achieves higher training reward than the other baselines, indicating that it learns a more effective policy update direction during optimization.
At the same time, PARPO also maintains superior training and validation success rates, suggesting that its advantage is not limited to fitting the training environment but can also transfer better to held-out evaluation settings.

From a behavioral perspective, PARPO also achieves a higher tool-call success rate.
This suggests that PARPO improves not only the final task completion outcome, but also the intermediate capabilities required for effective interaction with the environment, especially in scenarios that require proactive decision-making and external tool usage.
In addition, while achieving better performance, PARPO keeps the KL divergence low and stable, indicating that the performance gain does not come from excessively deviating from the initial policy, but rather from more stable and efficient policy improvement.

Overall, the ETAPP-Hard results are consistent with the observations on ETApp in the main text: PARPO shows stronger performance in training efficiency, optimization stability, and final success rate.
This suggests that the advantage of PARPO is not limited to the standard setting, but can also generalize to more difficult proactive task-oriented interaction scenarios.

\begin{figure*}[t]
    \centering
    \includegraphics[width=\textwidth]{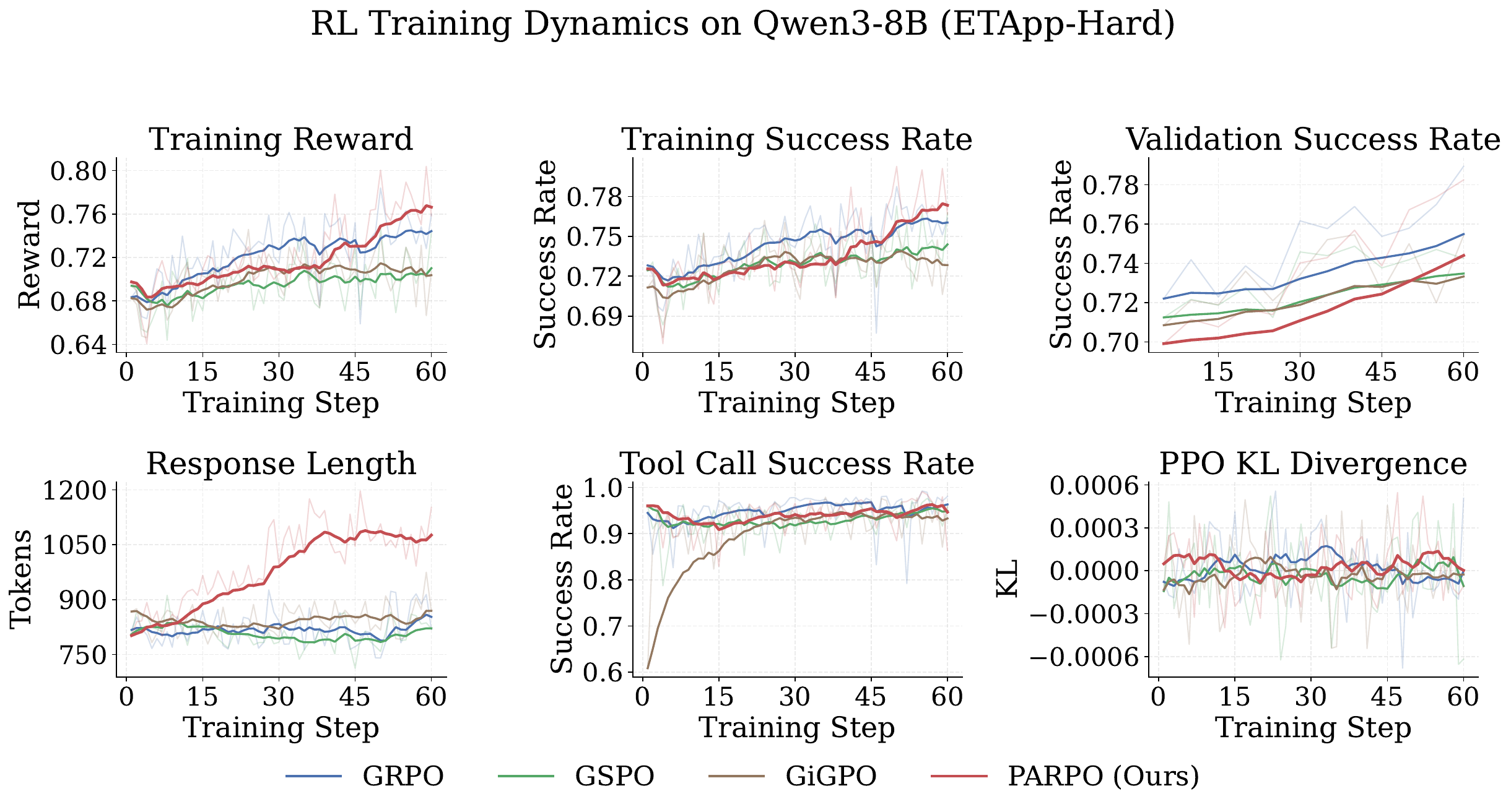}
    \caption{RL training dynamics of Qwen3-8B on ETAPP-Hard, comparing GRPO, GSPO, GiGPO, and PARPO.
    Consistent with the main text, all RL training experiments in this section are conducted without using our personalized reward model or Skill Graph.}
    \label{fig:rl_training_dynamics_etapp_hard}
\end{figure*}

\subsection{Effect of Personalized Reward Optimization on ETAPP-Hard}
\label{app:hard_reward_optimization}

We further analyze the effect of personalized reward optimization on ETAPP-Hard.
As in the main text, we compare GRPO, GSPO, GiGPO, and PARPO, and decompose the training process according to the four reward dimensions defined in ETApp, namely personalization, judge, proactivity, and procedure.
Here, \textit{judge} is a weighted aggregate metric computed from multiple evaluation criteria, and is therefore better interpreted as a summary of overall response quality rather than a single atomic capability.

As shown in Figure~\ref{fig:reward_decomposition_etapp_hard}, PARPO continues to show consistent advantages across all four reward dimensions under the more difficult ETAPP-Hard setting.
Although the absolute scores of all methods are lower than those on standard ETApp due to the increased difficulty, PARPO still achieves higher personalization, proactivity, and procedure scores throughout training, while the judge metric also shows a better overall trend.
These observations suggest that the benefits of PARPO are not confined to one isolated aspect, but instead lead to more robust improvements across multiple complementary dimensions.

More specifically, the improvement on personalization remains the most pronounced, indicating that PARPO is still better able to learn behaviors aligned with user preferences and contextual needs under challenging scenarios.
In addition, the sustained gain on proactivity suggests that PARPO is more effective at helping the model learn when to take initiative, when to invoke tools, and how to advance the interaction process more appropriately.
PARPO also achieves better results on the procedure dimension, indicating that it improves not only the final output quality, but also the organization and execution quality of the intermediate steps used to complete the task.

Figure~\ref{fig:reward_dimension_comparison_etapp_hard} further summarizes the final EMA scores at the last training step.
PARPO achieves the best results across all four reward dimensions, with the most notable advantage on personalization.
Meanwhile, the improvement on the judge score, as a weighted aggregate metric, also reflects a consistent gain in overall response quality.
Taken together, both the training curves and the final summary results show that the benefits of personalized reward optimization extend beyond the standard setting and remain effective in the more complex ETAPP-Hard scenario.

Overall, these results indicate that explicitly optimizing decomposed personalized rewards provides a more effective learning signal in difficult interactive environments.
Its advantage is reflected not only in the final scores, but also in more stable training, more balanced improvements across reward dimensions, and stronger proactive behaviors, further demonstrating the effectiveness and generalization ability of PARPO.

\begin{figure*}[t]
    \centering
    \includegraphics[width=0.92\textwidth]{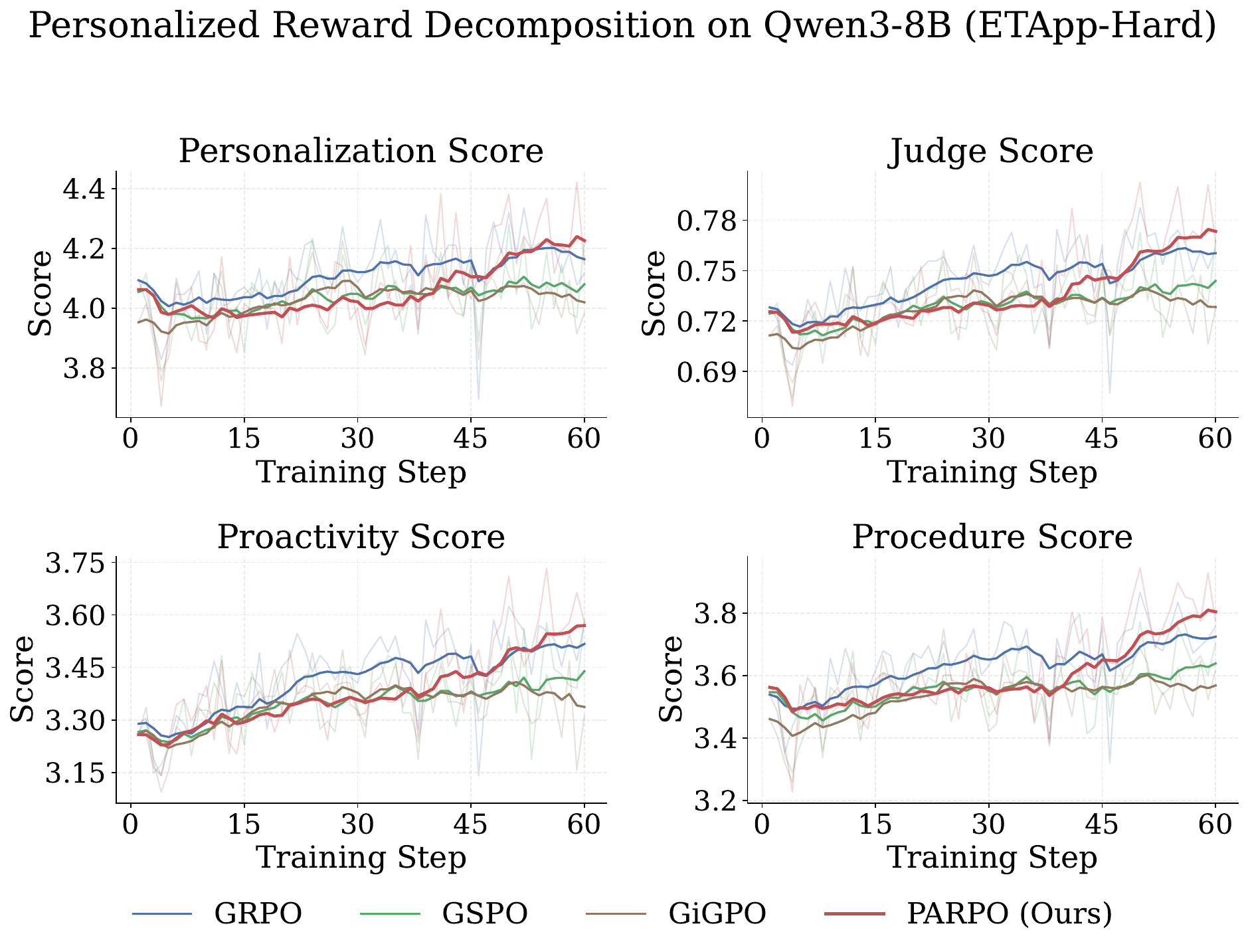}
    \caption{Personalized reward decomposition of Qwen3-8B on ETAPP-Hard, comparing GRPO, GSPO, GiGPO, and PARPO.}
    \label{fig:reward_decomposition_etapp_hard}
\end{figure*}

\begin{figure}[t]
    \centering
    \includegraphics[width=0.8\linewidth]{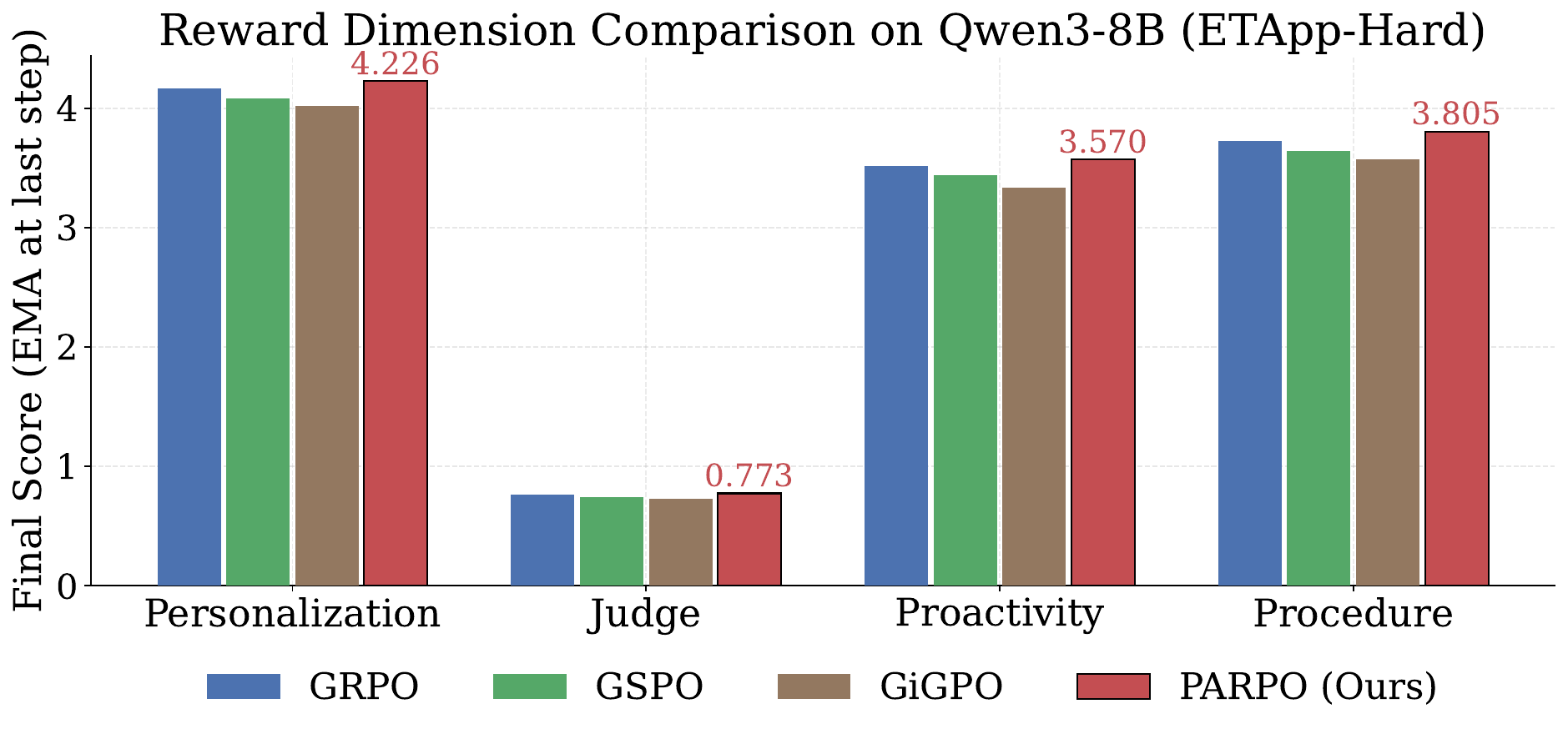}
    \caption{Final EMA scores at the last training step across different reward dimensions on ETAPP-Hard.}
    \label{fig:reward_dimension_comparison_etapp_hard}
\end{figure}

\section{Broader Impacts}
\label{app:broader_impacts}
This work studies personalized Agentic RL for user-conditioned tasks, aiming to improve how agents adapt to diverse user preferences in applications such as e-commerce assistance, travel planning, and daily scheduling. A potential positive societal impact is that personalized agents may provide more relevant and user-aligned recommendations, improving usability and decision support quality.

At the same time, this work may introduce risks. Personalized optimization depends on modeling user preferences and interaction histories, which may raise privacy concerns. Inaccurate or biased preference modeling may also reinforce stereotypes, amplify filter bubbles, or produce unfair recommendations across users. In addition, in domains such as e-commerce, personalized agents could be misused to optimize for engagement in manipulative ways rather than users' long-term interests.

To mitigate these risks, future deployments should incorporate privacy protection, user control over preference data, bias auditing across user groups, and monitoring for harmful or manipulative behavior. We hope this work encourages future research on personalized agents that are not only more effective, but also safer, fairer, and more transparent.

\section{Case Study: Shopping-Cart Alternatives}
\label{app:case_shopping_cart}

\tcbset{
    enhanced,
    boxrule=0.8pt,
    arc=2mm,
    boxsep=2pt,
    top=3pt,
    bottom=3pt,
    left=4pt,
    right=4pt,
    before skip=4pt,
    after skip=4pt,
}

\begin{tcolorbox}[
    colback=orange!8,
    colframe=orange!70!black,
    title=\textbf{Question},
    fonttitle=\bfseries
]
\textbf{User Query.}
\emph{Take a look at my shopping cart and recommend some cost-effective alternatives for this type of product.}

\smallskip
\textbf{User Context.}
Jordan Carter is a 28-year-old social media blogger in Los Angeles. Her profile suggests preferences for trendy and aesthetically appealing products, vegan-friendly lifestyle choices, travel gear, fashion, and tech accessories.
\end{tcolorbox}

\begin{tcolorbox}[
    colback=red!4,
    colframe=red!60!black,
    title=\textbf{PARPO (Ours)},
    fonttitle=\bfseries
]
Here is your personalized cost-effective alternatives guide.

\textbf{Current cart total:} \$374.95. \textbf{Potential savings:} up to \$117.12.

\begin{itemize}[leftmargin=1.4em, itemsep=0.15em, topsep=0.2em, parsep=0pt, partopsep=0pt]
    \item \textbf{Smart Home Speaker (\$129.99):} recommends Amazon Echo alternatives priced at \$49.99 and \$54.99, highlighting both savings and aesthetic fit for Jordan's home office.
    \item \textbf{Organic Cotton Bed Sheets (\$59.99):} suggests lower-cost cotton or microfiber options such as UTOPIA Bedding and Mellanni, with estimated savings of \$25--\$37.
    \item \textbf{The Midnight Library (\$14.99):} suggests Kindle, audiobook, or similar literary alternatives with lower prices and better portability.
\end{itemize}

The response explicitly connects recommendations to Jordan's lifestyle, including home aesthetics, travel convenience, and day-to-day content creation.
\end{tcolorbox}

\begin{tcolorbox}[
    colback=green!4,
    colframe=green!50!black,
    title=\textbf{SkillRL},
    fonttitle=\bfseries
]
The system reports that tool-based search does not return valid results, then falls back to practical recommendations based on Jordan's profile.

\begin{itemize}[leftmargin=1.4em, itemsep=0.15em, topsep=0.2em, parsep=0pt, partopsep=0pt]
    \item Confirms the shopping cart contents.
    \item Recommends Amazon Echo Dot as a lower-cost smart-speaker alternative.
    \item Suggests vegan sneaker alternatives from Reformation or Pact.
    \item Suggests lower-cost travel backpack alternatives from Columbia or Patagonia.
\end{itemize}

The response also states that these options remain compatible with Jordan's preference for sustainable, trendy, and content-creation-friendly products.
\end{tcolorbox}

\begin{tcolorbox}[
    colback=blue!4,
    colframe=blue!55!black,
    title=\textbf{GPT-4o},
    fonttitle=\bfseries
]
Here is a breakdown of the items in your cart and some lower-cost alternatives.

\begin{itemize}[leftmargin=1.4em, itemsep=0.15em, topsep=0.2em, parsep=0pt, partopsep=0pt]
    \item \textbf{Smart speaker:} suggests Anker or Echo Dot alternatives in the \$50--80 range.
    \item \textbf{Bed sheets:} suggests Mellanni or Amazon Basics in the \$30--40 range.
    \item \textbf{Book:} suggests eBook or second-hand versions at lower prices.
    \item \textbf{Vegan sneakers:} suggests NAE Vegan Shoes or Will's Vegan Store.
    \item \textbf{Travel backpack:} suggests Lowepro or Amazon Basics options in the \$40--60 range.
\end{itemize}

The response ends by asking whether the user wants more specific alternatives for any item.
\end{tcolorbox}

\begin{tcolorbox}[
    colback=gray!8,
    colframe=black!60,
    title=\textbf{Claude 3.5 Sonnet},
    fonttitle=\bfseries
]
The assistant says that the current shopping cart has not been viewed or accessed in the provided context. It asks the user to clarify:

\begin{itemize}[leftmargin=1.4em, itemsep=0.15em, topsep=0.2em, parsep=0pt, partopsep=0pt]
    \item what product or category is in the cart,
    \item the desired budget range,
    \item and preferred brand or quality level.
\end{itemize}

It then states that, once such information is provided, it can recommend lower-cost options aligned with Jordan's preferences.
\end{tcolorbox}

\newpage

\end{document}